\documentclass[twoside]{article}

\usepackage[accepted]{aistats2019_arxiv}
\usepackage[round]{natbib}

\usepackage{hyperref}       % hyperlinks
\usepackage{booktabs}       % professional-quality tables
\usepackage{amsfonts}       % blackboard math symbols
\usepackage{nicefrac}       % compact symbols for 1/2, etc.
\usepackage{microtype}      % microtypography
\usepackage{url}            % simple URL typesetting
\usepackage{amsmath}
\usepackage{amssymb}
\usepackage{amsthm}
\usepackage{natbib} % citations
\usepackage{graphicx} 
\usepackage{fancyhdr}
\usepackage{times}
\usepackage{algorithm} % Algorithm
\usepackage{algorithmic} 
\usepackage{subfigure} 
\usepackage{multirow}
\usepackage{here}
\usepackage{wrapfig}
\usepackage{centernot}
\usepackage{enumitem}
\usepackage{authblk}
\usepackage{breakcites}

\usepackage{comment}

\newtheorem{theorem}{Theorem}

\newtheorem{proposition}{Proposition}

\newtheorem{corollary}{Corollary}
\newtheorem{assumption}{Assumption}

\newcommand{\defeq}{\overset{def}{=}}

\newcommand{\sgn}{{\rm sgn}}

\newcommand{\argmin}{{\rm arg}\min}
\newcommand{\supp}{{\rm supp}}

\def\pd<#1>{\left\langle #1 \right\rangle}
\def\floor[#1]{\left\lfloor #1 \right\rfloor}
\def\ceil[#1]{\left\lceil #1 \right\rceil}

\newcommand{\realsp}{\mathbb{R}}

\newcommand{\expec}{\mathbb{E}}
\newcommand{\prob}{\mathbb{P}}

\newcommand{\hilsp}{\mathcal{H}}
\newcommand{\featuresp}{\mathcal{X}}
\newcommand{\labelsp}{\mathcal{Y}}
\newcommand{\tpr}{\rho}

\def\error{\mathcal{R}}
\def\risk{\mathcal{L}}

\newif\ifWITHSUPP
\WITHSUPPtrue % TRUE
\begin{document}

\runningtitle{Stochastic Gradient Descent with Exponential Convergence Rates of Expected Classification Errors}

\runningauthor{Atsushi Nitanda, Taiji Suzuki}

\twocolumn[

\aistatstitle{Stochastic Gradient Descent with Exponential Convergence Rates of \\Expected Classification Errors}

\aistatsauthor{ Atsushi Nitanda$^{\dag,1,2}$ %\\ \texttt{atsushi\_nitanda@mist.i.u-tokyo.ac.jp}
\And Taiji Suzuki$^{\ddag,1,2}$ } %\\ \texttt{taiji@mist.i.u-tokyo.ac.jp}}

%\aistatsauthor{ Author 1 \And Author 2}
%\aistatsaddress{ Institution 1 \And  Institution 2} ]

\aistatsaddress{
$^1$Graduate School of Information Science and Technology, The University of Tokyo\\
$^2$Center for Advanced Intelligence Project, RIKEN } ]

\begin{abstract}
We consider stochastic gradient descent and its averaging variant for binary classification problems in a reproducing kernel Hilbert space.
In traditional analysis using a consistency property of loss functions, it is known that the expected classification error converges more slowly than the expected risk even when assuming a low-noise condition on the conditional label probabilities.
Consequently, the resulting rate is sublinear. 
Therefore, it is important to consider whether much faster convergence of the expected classification error can be achieved.
In recent research, an exponential convergence rate for stochastic gradient descent was shown under a strong low-noise condition but provided theoretical analysis was limited to the squared loss function, which is somewhat inadequate for binary classification tasks.
In this paper, we show an exponential convergence of the expected classification error in the final phase of the stochastic gradient descent for a wide class of differentiable convex loss functions under similar assumptions.
As for the averaged stochastic gradient descent, we show that the same convergence rate holds from the early phase of training.
In experiments, we verify our analyses on the $L_2$-regularized logistic regression.
\end{abstract}

\section{Introduction}
The ultimate goal of binary classification problems is to find the Bayes classifier that minimizes the expected classification error in the space of all measurable functions.
Usually, this goal is achieved by approximating the classification error with a convex surrogate loss function and solving the expected risk minimization problem defined by this surrogate loss.
Such an approximation is theoretically justified by the consistency property \citep{zhang2004statistical,bartlett2006convexity} of loss functions,
which gives the connection between the {\it excess risk} (equivalent to the expected risk minus the Bayes risk which is the minimum expected risk over all measurable functions) and
the {\it excess classification error} (equivalent to the expected classification error minus the Bayes classification error which is the error of Bayes classifier).

Stochastic gradient descent \citep{robbins1951stochastic} is the workhorse method for large-scale machine learning problems, including the binary classification, owing to its scalability,
wide applicability for various problems, simplicity of implementation, and superior performance.
Hence, there is a great deal of research into improving its performance and analyzing the convergence behavior under various problem settings.
A popular variant of the method is averaged stochastic gradient descent \citep{ruppert1988efficient,polyak1992acceleration,rakhlin2012making,lacoste2012simpler}, which returns a weighted average of parameters
obtained by stochastic gradient descent to stabilize the iterates. 
Moreover, these methods have been generalized into a kernel setting \citep{cesa2004generalization,smale2006online,ying2006online}.
The convergence rates for the expected risk minimization have been also well studied. 
For instance, the rates of $O(1/\sqrt{T})$ and $O(1/T)$, where $T$ is the number of iterations, were obtained in
\cite{nemirovski2009robust,bach2011non,rakhlin2012making,lacoste2012simpler,ghadimi2013stochastic,bubeck2015convex,bottou2018optimization} 
for the general convex and strongly convex problems, and these rates are known to be asymptotically optimal  \citep{nemirovskii1983problem,agarwal2009information}.
%These results also remind us of comparable performance of stochastic gradient descent with the empirical risk minimizer of the convex risk as an estimator.
Note that convergence rates of excess classification errors can be simply derived from these rates with the consistency property of loss functions and can be accelerated by some preferable assumptions such as
the low-noise condition \citep{tsybakov2004optimal,bartlett2006convexity} on the conditional probability of the class label (c.f., \cite{ying2006online}), 
but obtained rates of excess classification errors are generally slower than those of excess risk functions.

In \cite{audibert2007fast,koltchinskii2005exponential}, it is shown that the convergence rate of the excess classification error of the empirical risk minimizer can be exponentially fast by assuming the {\it strong low-noise condition}
that conditional label probabilities are uniformly bounded away from $1/2$, although the excess risk converges at sublinear rate.
This is a rather surprising result because exponential convergence is significantly faster than sublinear convergence.
However, these theories are insufficient to explain the great success of stochastic gradient descent.
More recently, \cite{pillaud2017exponential} has provided direct analysis concerning an exponential convergence property of
stochastic gradient descent in a reproducing kernel Hilbert space, but
\cite{pillaud2017exponential} adopts the squared loss function, which is somewhat inadequate for the binary classification problems \citep{rosasco2004loss}.

\paragraph{Our contribution} In this paper, we extend the results in \cite{pillaud2017exponential} to general loss functions which are more appropriate for classification problems
by utilizing a different strategy of the proof from the one in \cite{pillaud2017exponential}.
That is, we show the exponential convergence of the excess classification error in the final phase of the learning procedure
using stochastic gradient descent for binary classification problems defined by a wide class of differentiable classification loss functions,  
including the logistic loss and the exponential loss.
Since, a method considered in our analysis corresponds to the common form of stochastic gradient descent, the traditional sublinear convergence rates of $O(1/T^q)$ ($q\in (0,1)$) also hold
in the overall learning procedure.
In that sense, our result implies acceleration of the excess classification error in the final phase of stochastic gradient descent.
In addition, we show a much better convergence result for the averaged stochastic gradient which reduces a threshold for the beginning time (the number of iteration) of exponential convergence.
As a result, we conclude that the averaged stochastic gradient exhibits exponential convergence from the early phase of the learning procedure in practice.
Moreover, an obtained convergence rate is the same as that in \cite{pillaud2017exponential} for the squared loss function.
Although, these results may be further improved by making an additional assumption on a decreasing rate of eigen-values of the covariance operator 
as shown in \cite{pillaud2017exponential}, we do not treat it in this study for the simplicity.
%However, we note that they also provided a faster convergence rate by making an additional assumption which is not treated in this study.
%Moreover, an obtained convergence rate is the same as that in \cite{pillaud2017exponential} for the squared loss function without an additional assumption 
%on a decreasing rate of eigen-values of the covariance operator, which may further improves a convergence rate as shown in \cite{pillaud2017exponential}.
Namely, we generalize the result in \cite{pillaud2017exponential} without degradation under general settings.

\paragraph{Technical difficulties} We next explain our technical contributions.
An obtained result in this work is a generalization of \cite{pillaud2017exponential} and an outline of the proof is essentially the same as that in \cite{pillaud2017exponential}, 
but we emphasize that we use proof techniques that were not argued in \cite{pillaud2017exponential} to overcome several obstacles caused by generalization of loss functions, so that details of the proof are quite different.
For instance, the stability argument \citep{bousquet2002stability,hardt2016train,liu2017algorithmic} of stochastic gradient descent is used to bound the error term of it and the property of the link function \citep{zhang2004statistical} is used to show the convergence of the expected classification error from the convergence of the stochastic gradient descent.
As a result, (i) obtained convergence rates in this study are much faster than that derived from another concentration inequality 
such as \cite{kakade2009generalization}
and (ii) the overall proof is significantly simplified and shortened without degradation under general settings compared to that of \cite{pillaud2017exponential} which relies on the specific update rule for the squared loss.

We finally note that exponential convergence of the stochastic gradient descent cannot be obtained from analyses \citep{audibert2007fast,koltchinskii2005exponential} because these work does not improve the bound of the consistency property of loss functions but only analyze the convergence rate of the empirical risk minimizer.
In addition, it is also generally difficult to show that the stochastic gradient descent has the comparative convergence rate to the empirical risk minimizer, 
thus an analysis of the stochastic gradient descent cannot be reduced to that of the empirical risk minimizer.
Even when such an argument is valid, analyses given in \citep{audibert2007fast,koltchinskii2005exponential} cannot be utilized under our problem setting 
because \cite{audibert2007fast} focuses on a more specific model of local polynomial estimators and \cite{koltchinskii2005exponential} assumes additional assumptions such as the Lipschitz continuity of hypotheses.

\begin{comment}
\paragraph{The organization of the paper}
In Section \ref{sec:problem_setup}, we give a problem setting of the binary classification problem on a reproducing kernel Hilbert space.
In Section \ref{sec:sgd}, we describe the vanilla stochastic gradient descent and its averaging variant analyzed in this paper.
In Section \ref{sec:analyses}, we first give the definition of the strong low-noise condition and several assumptions with their reasonability to guarantee the fast convergence of the expected classification error by the stochastic gradient descent.
Moreover, an exponential convergence rates of the stochastic gradient descent and a further improved rate of the averaged stochastic gradient descent are provided.
In the latter half of this section, we explain the main idea of the proof using the stability argument for stochastic gradient descent and the property of the link function of the loss.
Finally, our theory is validated in Section \ref{sec:experiments}.
\end{comment}

\section{Problem Setting} \label{sec:problem_setup}
In this section, we provide notations to describe a problem setting for the binary classification treated in this paper.
Let $\featuresp$ and $\labelsp$ be a measurable feature space and the set of binary labels $\{-1,1\}$, respectively.
We denote by $\tpr$ a probability measure on $\featuresp \times \labelsp$, by $\tpr_{\featuresp}$ the marginal distribution on $X$, and by $\tpr(\cdot| X)$ the conditional distribution on $Y$,
where $(X,Y)\sim\tpr$.
The ultimate goal in binary classification problems is to find a classifier $g: \featuresp \rightarrow \realsp$ such that $\sgn( g(x) )$ correctly classifies its label.
In other words, we want to obtain the Bayes classifier derived from $g(x)=\expec[Y|x] = 2\tpr(1|x) - 1$
that minimizes the expected classification error $\error(g)$ defined below over all measurable functions:
\begin{equation} \label{eq:expected_classification_error}
  \error(g) \defeq \expec_{(X,Y)}[I( \sgn(g(X)),Y)], 
\end{equation}
where the expectation is taken with respect to $(X,Y)\sim \tpr$. Here, $I$ is the $0$-$1$ error function: 
\begin{align*}
I(y,y') \defeq \left\{ \begin{array}{ll}
1 & (y \neq y'), \\
0 & (y = y'). 
\end{array} \right.
\end{align*}

However, since minimizing $\error(g)$ is intractable due to its non-continuity and non-convexity,
we approximate the problem with a convex surrogate loss function $l(\zeta,y)$ for the $0$-$1$ error function and minimize the expected risk defined by this surrogate loss function over a given hypothesis class of functions from $\featuresp$ to $\realsp$.
In general, a loss function $l$ has a form: $l(\zeta,y)=\phi(y\zeta)$ where $\phi$ is a non-negative convex function from $\realsp$ to $\realsp$.
Typical examples of such functions are $\phi(v) = \log(1+\exp(-v))$ for logistic regression and $\phi(v) = \exp(-v)$ for Adaboost.
We sometimes denote $z=(x,y)$ and $Z=(X,Y)$ for simplicity.
In this paper, we consider a reproducing kernel Hilbert space (RKHS) $(\hilsp_k,\pd<,>_{\hilsp_k})$ associated with a real-valued kernel function $k$ on $\featuresp$ as a hypothesis class, 
and denote by $\|\cdot\|_{\hilsp_k}$ the norm induced by the inner product $\pd<,>_{\hilsp_k}$.
As a result, the problem to be solved takes the following form:
\begin{equation} \label{eq:expected_risk_minimization}
  \min_{g \in \hilsp_k} \left\{ \risk_\lambda(g) \defeq \expec_Z[l(g(X),Y)] + \frac{\lambda}{2}\|g\|_{\hilsp_k}^2 \right\}.
\end{equation}
where the last term is the $L_2$-regularization in $\hilsp_k$ with a regularization parameter $\lambda > 0$.
The purpose of the regularization in this paper is to accelerate and stabilize the stochastic gradient descent to solve this expected risk minimization problem.
We also denote $\risk(g)=\risk_0(g)$.
Remark that although stochastic gradient descent is used to solve the problem (\ref{eq:expected_risk_minimization}),
the main interest in this paper is the convergence rate of the expected classification error (\ref{eq:expected_classification_error}).

\section{Stochastic Gradient Descent and its Averaging Variant in RKHS} \label{sec:sgd}
Stochastic gradient descent and its averaging variant are the most popular methods for solving large-scale machine learning problems.
In this paper, we analyze the convergence behavior of the expected classification error for these methods.
To do so, we give specific form of (averaged) stochastic gradient descent based on \cite{bottou2018optimization,lacoste2012simpler} for solving the problem.
We first recall the definition of a gradient of a function $F$ on $\hilsp_k$ at $g \in \hilsp_k$; it is an element $\nabla F(g)$ of $\hilsp_k$ satisfying the following equation:
for $\forall h \in \hilsp_k$,
\begin{equation*}
  F(g+h) = F(g) + \pd<\nabla F(g), h>_{\hilsp_k} + o(\|h\|_{\hilsp_k}).
\end{equation*}
For the expected risk $\risk$, when $k(x,x)$ is bounded on $\featuresp$, its gradient exists and takes the form $\expec[ \partial_\zeta l(g(X),Y)k(X,\cdot) ]$
($\zeta$ is the first variable of $l$), which is confirmed by the following equations:
\begin{align*}
\expec[ &l((g+h)(X),Y)] \\
 &= \expec[ l(g(X),Y) + \partial_\zeta l(g(X),Y) h(X) + o(|h(X)|) ],    
\end{align*}
$h(X)=\pd<h,k(X,\cdot)>_{\hilsp_k}$, and $|h(X)|\leq \|h\|_{\hilsp_k}\sqrt{k(X,X)}$.
Thus, the stochastic gradients of $\risk$ and $\risk_\lambda$ are given by $\partial_\zeta l(g(X),Y)k(X,\cdot)$ and $\partial_\zeta l(g(X),Y)k(X,\cdot) + \lambda g$ 
for $(X,Y)\sim \tpr$.
We denote by $G_\lambda(g,Z)$ the latter stochastic gradient.
% Using $G_\lambda(g,Z)$, stochastic gradient descent is performed, whose concrete description is given in Algorithm \ref{alg:sgd}.
Stochastic gradient descent is described in Algorithm \ref{alg:sgd}.
We can also use averaged stochastic gradient descent that returns a weighted average of obtained iterates $g_t$, rather than the last iterate $g_{T+1}$.
We denote by $\overline{g}_{T+1}=\sum_{t=1}^{T+1}\alpha_t g_t$.

\begin{algorithm}[ht]
  \caption{Stochastic Gradient Descent with the Averaging Option}
  \label{alg:sgd}
\begin{algorithmic}
  \STATE {\bfseries Input:}
  number of outer-iterations $T$,
  regularization parameter $\lambda$,
  learning rates $(\eta_t)_{t=1}^T$,
  averaging weights $(\alpha_t)_{t=1}^{T+1}$,
  initial function $g_1$
\\
\vspace{1mm}
   \FOR{$t=1$ {\bfseries to} $T$}
   \STATE Randomly draw a sample $z_t=(x_t,y_t) \sim \tpr$\\
   \STATE $g_{t+1} \leftarrow g_t - \eta_t G_\lambda(g_t,z_t)$ \\
   \ENDFOR
   \STATE Return $g_{T+1}$ or $\sum_{t=1}^{T+1}\alpha_t g_t$ (averaging option)
\end{algorithmic}
\end{algorithm}

In this paper, we adopt the following decreasing learning rate and averaging weight:
\begin{equation*} % \label{eq:learning_rates}
  \eta_t = \frac{2}{\lambda(\gamma + t)}, \hspace{3mm}\alpha_t = \frac{2(\gamma+t-1)}{(2\gamma +T)(T+1)},
\end{equation*}
where $\gamma$ is an offset parameter for the time index.
This learning rate is also used in \cite{bottou2018optimization}.
As for an averaging weight, it is a modified version of that introduced in \cite{lacoste2012simpler}.
We note that an averaged iterate $\overline{g}_{t}$ can be obtained in an iterative fashion as follows: $\overline{g}_1 = g_1$ and
\begin{equation*} %\label{eq:average}
  \overline{g}_{t+1} \leftarrow (1-\beta_t) \overline{g}_t + \beta_t g_{t+1}, \hspace{3mm} \beta_t = \frac{2(\gamma + t)}{ (t+1)(2\gamma + t)}.
\end{equation*}
Moreover, since this update does not require storing all internal iterates $(g_t)_{t=1}^{T+1}$, it is more efficient than taking the average of them as described in Algorithm \ref{alg:sgd}.

\section{Analyses} \label{sec:analyses}
To ensure the exponential convergence of these methods, we make several assumptions and provide several notations.
% We provide several notations if they are well defined here.
Recall that $\phi: \realsp \rightarrow \realsp$ is a non-negative convex function to define a loss function $l$.
We define the ``link function'' $h_*$ from $(0,1)$ to $\realsp$ as follows if it is well defined; for $\forall \mu \in (0,1)$,
\begin{equation*}
 h_*(\mu) \defeq \argmin_{h \in \realsp}\{ \mu \phi(h) + (1-\mu)\phi(-h)\}
\end{equation*}
and denote by $l_*(\mu)$ a corresponding value:
\begin{equation*}
 l_*(\mu) \defeq \min_{h \in \realsp}\{ \mu \phi(h) + (1-\mu)\phi(-h)\}.
\end{equation*}
The link function $h_*$ is well-defined for several loss functions; e.g., for logistic loss, it becomes $h_*(\mu)=\log(\mu/(1-\mu))$.
Since $l_*$ is concave \citep{zhang2004statistical}, the Bregman divergence for $l_*$ can be defined by 
\[ d_{l_*}(\eta_1,\eta_2) \defeq - l_*(\eta_2) + l_*(\eta_1) + l_*'(\eta_1)(\eta_2-\eta_1) . \]
Let $g_*$ be Bayes rule for $\risk$ that minimizes the $\risk$ over all measurable functions.

\begin{assumption} \ 
\begin{description} \label{assump:convergence}
\item{{\bf(A1)}} $\phi$ (and also $l(\cdot,y)$) is differentiable and convex.
  There exists $M > 0$ such that $|\partial_\zeta l(\zeta,y)| \leq M$.
$\risk(g)$ is $L$-smooth, that is, there exists $L>0$ such that for $\forall g, \forall h \in \hilsp_k$,
\[ | \risk(g+h) - \risk(g) - \pd<\nabla \risk(g), h>_{\hilsp_k} | \leq \frac{L}{2}\| h \|_{\hilsp_k}^2. \]
\item{{\bf(A2)}} Assume $\supp(\tpr_{\featuresp})=\featuresp$ and there exists $R>0$ such that $k(x,x) \leq R^2$ for $\forall x \in \featuresp$.
\item{{\bf(A3)}} The strong low-noise condition holds; $\exists \delta \in (0,1/2)$ such that $|\rho(1|x) - \frac{1}{2}| > \delta$, $\tpr_\featuresp$-almost surely.
%\item{{\bf(A4)}} $\tpr(1| X)$ takes values in $(0,1)$, $\tpr_\featuresp$-almost surely.
%  $h_*$ is well-defined, differentiable, and invertible over $(0,1)$.
%  Moreover, $h_{*}^{-1}$ is $L'$-Lipschitz continuous and 
%  \[ \sgn(\mu-1/2) = \sgn(h_*(\mu)). \]
\item{{\bf(A4)}} $\tpr(1| X)$ takes values in $(0,1)$, $\tpr_\featuresp$-almost surely.
  $h_*$ is well-defined, differentiable, monotonically increasing, and invertible over $(0,1)$.
  Moreover, it follows that
  \[ \sgn(\mu-1/2) = \sgn(h_*(\mu)). \]
\item{{\bf(A5)}} Bregman divergence $d_{l_*}$ derived by $l_*$ is positive, that is, $d_{l_*}(\eta_1,\eta_2) = 0$ if and only if $\eta_1=\eta_2$.
  For the expected risk $\risk$, a unique Bayes rule $g_{*}$ (up to zero measure sets) exists in $\hilsp_k$.  
\end{description}
\end{assumption}
Assumption {\bf(A1)} is common in the literature and valid for several loss functions, for instance, the logistic loss and the smoothed hinge loss.
The boundedness of kernel function {\bf(A2)} is also reasonable.
Indeed, Gaussian kernel is bounded by $1$ and continuous kernel functions are bounded when $\featuresp$ is compact.
%This boundedness leads to an important relationship $R\|\cdot\|_{\hilsp_k} \geq \|\cdot\|_{L_\infty}$ (the sup norm over $\featuresp$).
This boundedness leads to an important relationship between norms $\|\cdot\|_{\hilsp_k}$ and $\|\cdot\|_{L_\infty}$ (the sup norm over $\featuresp$) as follows.
Since $g(x)=\pd<g,k(x,\cdot)>_{\hilsp_k}$ for arbitrary function $g \in \hilsp_k$ and $k(x,\cdot) \in \hilsp_k$ by the definition of kernel function, we get
\begin{equation*} %\label{eq:norm_relationship}
  \|g\|_{L_\infty} = \sup_{x \in \featuresp} |g(x)|
  \leq \|g\|_{\hilsp_k}\|k(x,\cdot)\|_{\hilsp_k} \leq R \|g\|_{\hilsp_k}.
\end{equation*}
%The latter part of {\bf(A2)} holds when $\featuresp$ is a topological space, $\tpr_\featuresp$ is strictly positive (i.e., an open set has a positive measure),
%and $k$ is continuous on $\featuresp \times \featuresp$ because any function in $\hilsp_k$ becomes continuous by the continuity of $k$.
The strong low-noise condition assumed in {\bf(A3)} is also adopted in \cite{koltchinskii2005exponential,audibert2007fast} to show the exponential convergence property of empirical risk minimizers for regularized problems.
More recently, \cite{pillaud2017exponential} exhibited exponential convergence of stochastic gradient descent for solving regularized least-squares regression for classification problems by using this condition.
We also note that this condition is the strongest version of more general low-noise conditions used in \cite{tsybakov2004optimal,bartlett2006convexity},
which also gives a faster convergence rate of generalization error of the empirical risk minimizer than $O(1/\sqrt{n})$, where $n$ is the number of training data,
but an exponential convergence is not achievable.
For logistic loss, since $h_*(\mu)=\log(\mu/(1-\mu))$ as introduced above, conditions in {\bf(A4)} are satisfied. %with $L'=1/4$ are satisfied.
In this setting, the Bregman divergence $d_{l_*}$ corresponds to the Kullback-Leibler divergence, which is positive.
It is known from \cite{zhang2004statistical} that the excess risk (that is the difference between the expected risk and the minimum expected risk over all measurable functions)
can be measured by Bregman divergence $d_{l_*}$ when $\phi, h_*$ are differentiable and $h_*$ is invertible:
\begin{equation} \label{eq:excess_risk}
  \risk(g) - \risk(g_*) = \expec_X[ d_{l_*}( h_*^{-1}(g(X)), \tpr(1|X) ) ].
\end{equation}
Therefore, if $d_{l_*}$ is positive, then Bayes rule $g_*(X)$ equals $h_*(\tpr(1|X))$, $\tpr_\featuresp$-almost surely.
Thus, the uniqueness of the Bayes rule for $\risk$ assumed in {\bf(A5)} is also verified for logistic loss. % and the other loss functions, too (see \cite{zhang2004statistical}).
% Note that by the consistency, $h_*^{-1}(g_*(X)) - 1/2$ is also Bayes rule for $\error$ for several loss functions, that is, it is a minimizer of $\error$ over all measurable functions.
Although we here focus on the logistic loss, Assumption {\bf(A1)}, {\bf(A4)}, and {\bf(A5)} are valid for other loss functions, such as squared loss and smoothed hinge loss. %see \cite{zhang2004statistical}.
Furthermore, when imposing a bounded convex constraint on the domain of the problem and assuming $\supp(\tpr_\featuresp)$ is bounded,
Assumption {\bf(A1)} can be relaxed to capture a more comprehensive class of loss functions, including exponential loss, which also satisfies {\bf(A4)} and {\bf(A5)}.
Even in this case, our analysis presented in the paper can be extended by considering projected stochastic gradient descent in an obvious way.
% although we make no such assumption for simplicity.

To describe our results, we introduce the following notation.
\[ m(\delta) \defeq \min\{ h_*(0.5 + \delta), | h_*(0.5-\delta)| \}. \]
This provides a lower bound on $|g_*(X)|$, i.e. $|g_*(X)| \geq m(\delta)$ almost surely.
For instance, for the logistic loss, $m(\delta) = \log((1+2\delta)/(1-2\delta))$ which converges to $\infty$ as $\delta \rightarrow 1/2$, resulting in 
better dependence on the low noise parameter in terms of the convergence rate.
Note that if $h_*^{-1}$ is $L'$-Lipschitz continuous, then $m(\delta) \geq \delta/L'$ for $\forall \delta \in (0,1/2)$, 
e.g., $L' = 1/4$ for the logistic loss and $L' = 1/2$ for the squared loss.

\subsection{Main Results}
Here, we describe our main results where stochastic gradient descent and averaged stochastic gradient descent converge to the Bayes rule for the expected classification error
with exponential convergence rates under sufficiently small $\lambda > 0$ guaranteeing sufficient closeness between $g_\lambda$ (minimizer of $\risk_\lambda$ in $\hilsp_k$) and $g_*$.
The existence of such a $\lambda$ will be shown later under Assumptions {\bf(A2)}--{\bf(A5)}.

Theorem \ref{thm:sgd_exp_convergence} is the main result for stochastic gradient descent.
Since the rate of $O(1/\sqrt{T})$ is optimal without the (strong) low-noise condition, it is rather surprising that a significantly fast rate such as an exponential convergence can be achievable.

\begin{theorem} [{\bf Exponential convergence rate for SGD}] \label{thm:sgd_exp_convergence}
Suppose Assumptions {\bf(A1)}--{\bf(A5)} hold.
There exists a sufficiently small $\lambda > 0$ such that the following statement holds.
Consider Algorithm \ref{alg:sgd} without the averaging option and with $\eta_t = 2/\lambda(\gamma + t)$,
where $\gamma$ is a positive value such that $\eta_1 \leq \min\{1/(L+\lambda),1/2\lambda\}$.
Let $\sigma^2>0$ be an upper-bound on the variance of stochastic gradient $\partial_\zeta l(g(X),Y)k(X,\cdot)$.
We assume $\|g_1\|_{\hilsp_k} \leq \left(2\eta_1+\frac{1}{\lambda}\right)MR$.
We set
\[ \nu \defeq \max\left\{ \frac{2}{\lambda^2}(L+\lambda)\sigma^2, (1+\gamma)(\risk_\lambda(g_1)- \risk_\lambda(g_\lambda))\right\}. \]
Then, for $T \geq \frac{32R^2\nu}{m^2(\delta)\lambda}-\gamma$, we have
  \begin{equation} \label{eq:sgd_exp_convergence}
    \expec[ \error(g_{T+1}) - \error([\expec[Y|x]) ] \leq 2\exp\left( - \frac{m^2(\delta)\lambda^2(\gamma+T)}{2^9 \cdot 9M^2R^4} \right) 
  \end{equation}
\end{theorem}

Note that the bound (\ref{eq:sgd_exp_convergence}) is valid only when $T$ is larger than the threshold of the time given in the theorem.
A similar threshold appeared in convergence results obtained in \cite{pillaud2017exponential} with better dependence on $\delta$, when $\delta \rightarrow 0$, 
than ours.
However, we remark that our analysis generalizes their results to reasonable smooth convex loss functions which is more natural than the squared loss for classification problems.
Moreover, since the low-noise parameter $\delta$ is a given fixed parameter, the dependency on $\delta$ is insignificant especially 
when $m(\delta)$ is rather large. % (e.g., $0.3$).
For instance, recall that $m(\delta)\rightarrow \infty$ as $\delta \rightarrow 1/2$ for the logistic loss, resulting in the faster convergence rate.
We moreover note that the threshold for the beginning time of exponential convergence is independent from a required precision for the excess classification error.
Therefore, the number of iterations $T$ needed to make the right hand side of (\ref{eq:sgd_exp_convergence}) smaller than a given precision exceeds the threshold,
if the precision is sufficiently small.
Furthermore, even when $T$ is smaller than the threshold, the convergence rate $O(1/T^q)$ ($q\in[0,1]$) of the expected classification error can be obtained
by the common analysis of the expected risk function and the consistency of the convex loss function.

We next give a main convergence result for averaged stochastic gradient descent which significantly reduces a time threshold compared to the vanilla stochastic gradient descent.
\begin{theorem} [{\bf Exponential convergence rate for averaged SGD}] \label{thm:asgd_exp_convergence}
Suppose Assumptions {\bf(A1)}--{\bf(A5)} hold.
There exists a sufficiently small $\lambda > 0$ such that the following statement holds.  
Consider Algorithm \ref{alg:sgd} with the averaging option, $\eta_t = 2/\lambda(\gamma + t)$, and $\alpha_t = 2(\gamma+t-1)/(2\gamma +T)(T+1)$,
where $\gamma$ is a positive value such that $\eta_1 \leq \min\{1/L,1/2\lambda\}$.
We assume $\|g_1\|_{\hilsp_k} \leq \left(2\eta_1+\frac{1}{\lambda}\right)MR$.
Then, for sufficiently large $T$ such that
\begin{equation*}
  \max \left\{\frac{36M^2R^2}{\lambda^2 (2\gamma + T)}, \frac{\gamma(\gamma-1)\|g_1-g_\lambda\|_{\hilsp_k}^2}{(2\gamma + T)(T+1)} \right\}
  \leq \frac{m^2(\delta)}{32R^2}, 
\end{equation*}
we have the following:
\begin{equation} \label{eq:asgd_exp_convergence}
  \expec[ \error(g_{T+1}) - \error([\expec[Y|x]) ] \leq 2\exp\left( - \frac{m^2(\delta)\lambda^2(2\gamma+T)}{2^{10}\cdot 9 M^2R^4} \right), 
\end{equation}
\end{theorem}

\paragraph{Remark} Although we assume {\bf(A1)}, Lipschitz smoothness of $\risk$ is not required in the proof of this theorem. That is, the parameter $L$ does not affect the performance of averaged stochastic gradient descent.

We notice that two convergence rates (\ref{eq:sgd_exp_convergence}) and (\ref{eq:asgd_exp_convergence}) are comparable, but the threshold of averaged stochastic gradient descent for the beginning time of exponential convergence has much better dependence on $\lambda$ than stochastic gradient descent, that is, 
the averaging technique accelerates the convergence in the early phase for small $\lambda$ as shown in \cite{rakhlin2012making,lacoste2012simpler}.
As a result, threshold on $T$ becomes not important.
From the convergence rate (\ref{eq:asgd_exp_convergence}), the required number of iterations to obtain $\epsilon$-accuracy is 
\begin{equation}
 O\left( \frac{1}{m^2(\delta) \lambda^2}\log\left( \frac{1}{\epsilon} \right)\right). \label{eq:asgd_complexity}    
\end{equation} 
We find clearly that this required iterations $T$ naturally exceeds the time threshold for a rather small $\epsilon$ when 
we ignore the second term in the maximum in the threshold which is often inactive.
In addition, this averaging scheme gives the sublinear convergence rate $O(1/T^q)$ ($q\in[0,1]$), even when $T$ is smaller than the threshold,
as explained in the case of Theorem \ref{thm:sgd_exp_convergence}.
Finally, we note that since $m(\delta)\geq \delta$ for the squared loss, a convergence rate (\ref{eq:asgd_exp_convergence}) is exactly the same as that obtained in \cite{pillaud2017exponential} when an additional assumption on the decreasing rate of eigenvalues of the covariance operator is not made.
In other words, we succeed in generalizing the result in \cite{pillaud2017exponential} without degradation under general settings.

As a corollary, we here derive a convergence rate for case of the logistic loss and Gaussian kernel,  
which can be obtained by setting $m(\delta)=\log((1+2\delta)/(1-2\delta))$ and $M,R=1$.
\begin{corollary}
Consider the logistic loss and Gaussian kernel.
Suppose the same assumptions in Theorem \ref{thm:asgd_exp_convergence} hold.
Then, there exists a sufficiently small $\lambda>0$ such that the following convergence rate of $\expec[ \error(g_{T+1}) - \error([\expec[Y|x]) ]$ holds. 
\begin{equation} \label{eq:asgd_exp_convergence_for_logloss_Gaussian}
  2\exp\left( - \frac{\lambda^2(2\gamma+T)}{2^{10}\cdot 9}\log^2\left(\frac{1+2\delta}{1-2\delta}\right) \right).
\end{equation}
\end{corollary}

\subsection{Proof Idea}
We here explain the proof idea for convergence theorems.
All missing proofs are found in the Appendix.
The proof is mainly composed of three parts.
% We denote $g_\lambda \defeq \argmin_{g \in \hilsp_k}\risk_\lambda(g)$.
We first show the Bayes optimality of $g_\lambda$ for a small $\lambda>0$ and specify the size of its neighborhood in $\hilsp_k$ to ensure the optimality.

\begin{proposition} \label{prop:bayes_region}
  Suppose {\bf(A2)}--{\bf(A5)} in Assumption \ref{assump:convergence} hold.
  Then, there exists $\lambda>0$ such that an arbitrary $g\in \hilsp_k$ satisfying $\|g - g_{\lambda}\|_{\hilsp_k} \leq m(\delta)/2R$ is the Bayes classifier of $\error(g)$. That is,
  $\error(\expec[Y|x]) = \error(g)$.
\end{proposition}

\paragraph{Remark} This proposition shows the existence of $\lambda$ to provide the Bayes classifier. 
In the expected risk minimization problem, such an appropriate value of $\lambda$ represents the inherent difficulty of the problem and that is controlled by the choice of kernel function $k:\featuresp\times \featuresp \rightarrow \realsp$.
For the infinite dimensional problems, specifying the value of $\lambda$ is somewhat difficult beforehand, indeed, it was not provided even for the squared loss \citep{pillaud2017exponential}.
However, we can specify $\lambda$ for finite dimensional problems as follows.
We assume that there exist positive values $\Delta,\ v_\Delta$ such that $\risk(g) \geq \risk(g_*) + v_\Delta$ for arbitrary $g$ satisfying $\|g-g_*\|_{\hilsp_k} \geq \Delta$. % (where $g_*$ is the Bayes rule of $\risk$).
This condition can be derived from the local strong convexity at $g_*$ which is often assumed for the logistic loss (c.f. \cite{bach2013non}).
Then, we can easily show that the minimizer $g_\lambda$ satisfies %of the regularized expected risk $\risk_\lambda(g) = \risk(g) + 0.5\lambda \|g\|_{\hilsp_k}^2$ satisfies
$\|g_\lambda - g_*\|_{\hilsp_k}< \Delta$ when $\|g_*\|_{\hilsp_k} \leq 2v_\Delta/\lambda$.
Therefore, $\lambda$ should be chosen to satisfy the condition $\|g_*\|_{\hilsp_k} \leq 2v_\Delta/\lambda$ for a target accuracy $\Delta=m(\delta)/2R$ as seen in the proof.
In short, an appropriate $\lambda$ depends on $\|g_*\|_{\hilsp_k}$ and the local strong convexity $v_\Delta$ at $g_*$.
As for the kernel $k$, it should be chosen for making them better conditioned under a bounded constraint $k(x,x)\leq R$.
As a result, the required sample size (number of iterations) depends on $\|g_*\|_{\hilsp_k}$ and the convexity $v_\Delta$ via the value of $\lambda$ (cf. Theorem \ref{thm:sgd_exp_convergence} and \ref{thm:asgd_exp_convergence}), but 
such a dependence is quite natural as seen in the theory of kernel ridge regression \citep{caponnetto2007optimal}.
%Finally, we note that the dependency of the required sample size (number of iterations) on $\|g_*\|_{\hilsp_k}$ and $\lambda$ as seen in Theorem \ref{thm:sgd_exp_convergence} and \ref{thm:asgd_exp_convergence} is quite natural as seen in the theory of kernel ridge regression \citep{caponnetto2007optimal}.

We notice that from Proposition \ref{prop:bayes_region}, the goal of classification problems is achieved by finding a function included in the neighborhood of $g_\lambda$ providing the Bayes rule for $\error$,
of which the existence is shown in the proposition.
Since $g_\lambda$ is the minimizer of $\risk_\lambda$ in $\hilsp_k$, it is expected that a sequence of iterates
obtained by a stochastic optimization method such as stochastic gradient descent to solve the problem converges to $g_\lambda$ with high probability.
To derive the probability and convergence rate to obtain the Bayes rule, we verify the convergence of an expected estimator and its variance.
For the variance, we utilize the following proposition to bound it.

\begin{proposition}[\cite{pinelis1994optimum}] \label{prop:martingale}
Let $D_1,\ldots,D_T$ be a martingale difference sequence taking values in $\hilsp_k$.
We assume that there exists a constant $c_T > 0$ such that $\sum_{t=1}^T\|D_t\|_\infty^2 \leq c_T^2$, where $\|D_t\|_\infty$ is the essential supremum of $\|D_t\|_{\hilsp_k}$.
Then, for $\forall \epsilon>0$, 
\begin{equation*}
  \prob \left[\sup_{T\geq \forall s \geq 1} \left\| \sum_{t=1}^s D_t \right\|_{\hilsp_k} \geq \epsilon \right] \leq 2\exp\left( -\frac{\epsilon^2}{2c_T^2}\right).
\end{equation*}
\end{proposition}
Let $\hat{g}_{T+1}$ stand for an output iterate of stochastic gradient descent or averaged stochastic gradient descent with $T$-iterations.
Let $Z_1,\ldots, Z_T$ be i.i.d. random variables following $\tpr$.
Since $D_t = \expec[\hat{g}_{T+1} | Z_1,\ldots,Z_t ] - \expec[\hat{g}_{T+1} | Z_1,\ldots,Z_{t-1} ]$ ($t\in \{1,\ldots,T\}$) is a martingale difference sequence,
Proposition \ref{prop:martingale} can be applied to bound the norm of the sum of $D_t$ over $t \in \{1,\ldots,T\}$.
Since $\sum_{t=1}^TD_t = \hat{g}_{T+1} - \expec[\hat{g}_{T+1}]$, we see that for $\forall \delta>0$,
with the probability at least $1-2\exp( - m^2(\delta)/32R^2c_T^2 )$,
\begin{equation} \label{eq:concentration}
  \left\| \hat{g}_{T+1} - \expec[\hat{g}_{T+1}] \right\|_{\hilsp_k} < \frac{m(\delta)}{4R}.
\end{equation}

Thus, by combining Proposition \ref{prop:bayes_region} and the inequality (\ref{eq:concentration}),
we conclude that if an expected function $\expec[\hat{g}_{T+1}]$ is in the neighborhood of the radius $\delta/4L'R$ around $g_\lambda$,
then $\hat{g}_{T+1}$ is the Bayes optimal for $\error$ with the probability at least $1-2\exp( - m^2(\delta)/32R^2c_T^2 )$.
That is, 
\[ \error(\hat{g}_{T+1}) = \error(\expec[Y|x]). \]
In other words, from the definition of the expected classification error $\error$, if $\left\| \expec[\hat{g}_{T+1}] - g_\lambda \right\|_{\hilsp_k} < m(\delta)/4R$, then 
\begin{equation} \label{eq:exp_convergence}
  \expec[ \error(\hat{g}_{T+1}) - \error(\expec[Y|x])] \leq 2\exp \left( - \frac{m^2(\delta)}{32R^2c_T^2}  \right).  
\end{equation}
Therefore, by confirming the convergence of $\expec[\hat{g}_{T+1}]$ to $g_\lambda$ and specifying the convergence rate $O(1/T^q)$ ($q>0$) of $c_T$ to zero,
we can conclude the exponential convergence of the expected classification error from the inequality (\ref{eq:exp_convergence}).
% Note that under the reasonable convergence of $\expec[\hat{g}_{T+1}]$,
% the convergence rate of $c_T$ becomes dominant for a sufficiently small required precision.
Thus, the remaining problem is to verify these convergences for Algorithm \ref{alg:sgd}.
As for the expected iterate $\expec[ \hat{g}_{T+1}]$, its convergence can be shown by naturally extending proofs \citep{bottou2018optimization,lacoste2012simpler}
for stochastic gradient descent in Euclidean space.
For $c_T$, we can show its convergence by utilizing an argument \citep{hardt2016train} to show the stability of stochastic gradient descent for strongly convex problems.
Such a combination of the martingale bound and the stability analysis of stochastic gradient descent has also been adopted in \cite{liu2017algorithmic} 
for another purpose.

\paragraph{Auxiliary results for main theorems}
We now exhibit auxiliary results for deriving the exponential convergence rate of stochastic gradient descent (Algorithm \ref{alg:sgd} without averaging).
%In this section, we derive a convergence rate of stochastic gradient descent (Algorithm \ref{alg:sgd}).
To do so, we here present convergence rates of quantities $\expec[g_{T+1}]$ and $c_T$. % needed to be specified.
The rate of the former is given in the following proposition.
\begin{proposition} \label{prop:convergence_sgd}
Suppose Assumption {\bf(A1)} holds.
Consider Algorithm \ref{alg:sgd} without averaging and with the same learning rates as in Theorem \ref{thm:sgd_exp_convergence}.
We assume that $\eta_1 \leq 1/(L+\lambda)$ and $\sigma^2>0$ is an upper-bound on the variance of stochastic gradient $\partial_\zeta l(g(X),Y)k(X,\cdot)$.
We set
\[ \nu \defeq \max\left\{ \frac{2}{\lambda^2}(L+\lambda)\sigma^2, (1+\gamma)(\risk_\lambda(g_1)- \risk_\lambda(g_\lambda))\right\}. \]
Then, we have
\[ \| \expec[g_{T}] - g_\lambda \|_{\hilsp_k}^2 \leq \frac{2\nu}{\lambda(\gamma + T)}. \]
\end{proposition}
This convergence can be shown in a standard way in the stochastic optimization literature.

On the other hand, a bound for the value of $c_T$ can be derived in the following manner.
Let $Z_t' \sim \tpr$ be a random variable independent from $Z_1,\ldots,Z_T$ and let $g_{T+1}^t$ be an output of stochastic gradient descent (Algorithm 1 without averaging)
depending on $(Z_1,\ldots,Z_{t-1},Z_t',Z_{t+1},\ldots,Z_T)$.
By setting $D_t = \expec[g_{T+1} | Z_1,\ldots,Z_t ] - \expec[g_{T+1} | Z_1,\ldots,Z_{t-1} ]$, we find
\begin{equation*}
  \|D_t\|_{\infty} \leq \expec[ \|g_{T+1} - g_{T+1}^t\|_{\infty} | Z_1,\ldots,Z_t ], 
\end{equation*}
where we recall that $\|\cdot\|_\infty$ is the essential supremum of $\|\cdot\|_{\hilsp_k}$.
Therefore, $c_T$ can be estimated by bounding $\|g_{T+1} - g_{T+1}^t\|_{\infty}$ uniformly with respect to random variables.
Such a bound would be obtained by the stability property \citep{hardt2016train}, that is, the small deviation that results from replacing one example for stochastic gradient descent.
As a result, this argument leads to the following proposition:

\begin{proposition} \label{prop:martingale_bound_sgd}
Suppose Assumptions {\bf(A1)} and {\bf(A2)} hold.
% Consider Algorithm \ref{alg:sgd} with the same setting as that in Theorem \ref{thm:sgd_exp_convergence}.
Consider Algorithm \ref{alg:sgd} without averaging and with the same learning rates as in Theorem \ref{thm:sgd_exp_convergence}.
We assume that $\eta_1 \leq \min\{ 1/L, 1/2\lambda\}$ and $\|g_1\|_{\hilsp_k} \leq \left(2\eta_1+\frac{1}{\lambda}\right)MR$.
Then, it follows that
\[ \sum_{t=1}^T \|D_t\|_{\infty}^2 \leq \frac{144M^2R^2}{\lambda^2(\gamma + T)},\]
where $D_t$ is a martingale difference, as defined previously.
\end{proposition}

By combining these two propositions in the way explained earlier,
we can prove the exponential convergence (Theorem \ref{thm:sgd_exp_convergence}) of the expected classification error for stochastic gradient descent.

We can also show Theorem \ref{thm:asgd_exp_convergence}, i.e., the exponential convergence of averaged stochastic gradient descent
by specifying the rate of $\expec[\overline{g}_{T+1}]$ to $g_\lambda$ and $c_T$ to zero for the algorithm.
Although the averaging method brings more preferable results as seen in Theorem \ref{thm:asgd_exp_convergence},
we defer auxiliary results for this theorem to the Appendix for simplicity.

\paragraph{Comparison with another concentration inequality}
We emphasize that our proof technique can provide a much faster convergence rate than 
that derived from another concentration inequality \citep{kakade2009generalization}.
Indeed, the following convergence rate of the objective gap was shown in \cite{kakade2009generalization} with probability at least $1-\log(T)q$,
\[ O\left( \frac{\log T}{\sqrt{T}}\right) + \frac{\sqrt{\log T}}{T} \sqrt{\log \left(\frac{1}{q}\right)} + \frac{\log(1/q)}{T}. \]
Therefore, $q$ should be $\exp(-o(T))$ to guarantee the convergence.
As a result, a convergence rate of expected classification error is $O(\log(T) \exp(-o(T)))$ which is much slower than our rates 
and a threshold on $T$ also has a worse trade-off with respect to the choice of $q$.

\section{Experiments} \label{sec:experiments}
In this section, we conduct numerical experiments on synthetic datasets to verify our theoretical analyses. 
The random Fourier feature \citep{rahimi2007random} is the most popular and widely used technique to approximate shift invariant kernels $k$ with
the dot-product: $\iota(x)\top\iota(x')$ ($\forall x, \forall x' \in \featuresp)$) through a non-linear embedding $\iota$
from $\featuresp$ to a low-dimensional Euclidean space $\realsp^D$.
When we use such a kernel defined by $\iota$, stochastic gradient descent and its averaging variant are reduced to those for linear models in an Euclidean space.
Moreover, since we assume that the Bayes rule $g_*$ is in $\hilsp_k$, it is reasonable for the numerical verification to consider linear models
and linear separable datasets in an Euclidean space in experiments.

\begin{figure}[ht]
% \vspace{-5mm}
\begin{center}
  \includegraphics[angle=0,width=50mm]{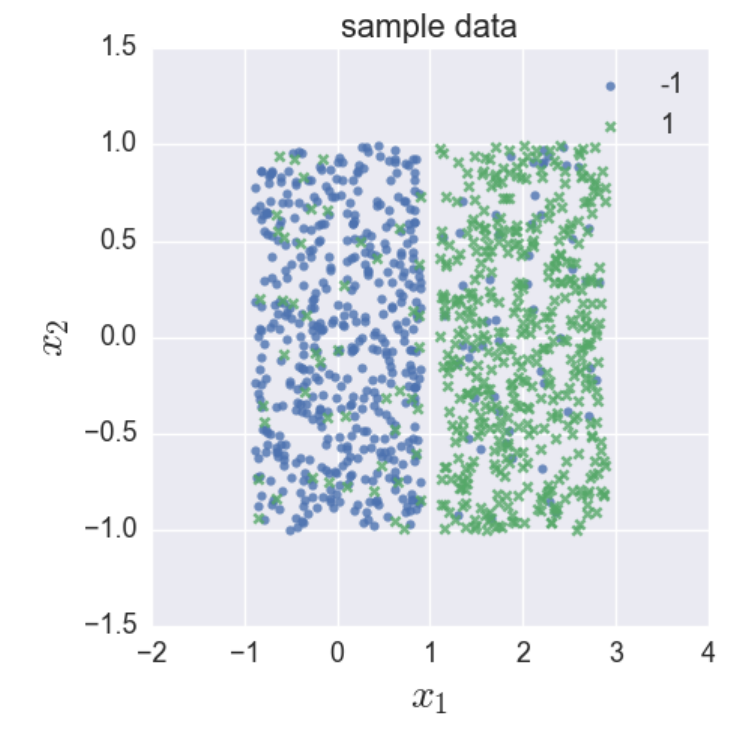}
\vspace{-5mm}
\caption{A subsample of the data used in the experiment with $\delta=0.4$.}  \label{fig:sample_data}
\end{center}
\vspace{-3mm}
\end{figure}

We here explain the experimental setting.
For the loss function, we use logistic loss.
For datasets, we use linear separable two dimensional synthetic datasets as shown in Figure \ref{fig:sample_data} which is subsampled from a dataset.
The support of datasets is composed of two part: $[-1+r,1-r]\times[-1,1]$  and $[1+r,3-r]\times[-1,1]$ in $\realsp^2$, where $r=0.1$.
We consider fixed conditional probabilities on each component, namely, for $\delta \in (0,0.5)$, we use $\tpr(Y=1|x)=0.5-\delta$ for the left part
and we use $\tpr(Y=1|x)=0.5+\delta$ for the right part of the support.
As for the low-noise parameter $\delta$, we test values from $\{0.1,0.25,0.4\}$ and as for the regularization parameter $\lambda$, we test values from $\{0.1,0.01,0.001,0.0001\}$.
Test datasets containing $100,000$ points are sampled from this distribution for each $\delta$.
We run stochastic gradient descent and averaged stochastic gradient descent $5$-times with $20,000$-iterations and we report the best results, on training accuracies,
with respect to $\lambda$ for each $\delta$.
Before running these methods, we additionally use $1000$-iterations for tuning hyperparameter $\gamma$ which is an offset for time index as the optimization proceeds well.
As for the regularization parameter $\lambda$, the value of $0.01$ is chosen for $\delta=0.1,0.25$ and the value of $0.0001$ is chosen for $\delta=0.4$.

\begin{figure*}[th]
  \begin{center}
   {\small
     \begin{tabular}{ccc}
  \hspace{-1mm} $\delta=0.4$ & \hspace{-5mm} $\delta=0.25$ & \hspace{-5mm} $\delta=0.1$ \\ 
  \hspace{-1mm} \includegraphics[angle=0,width=45mm]{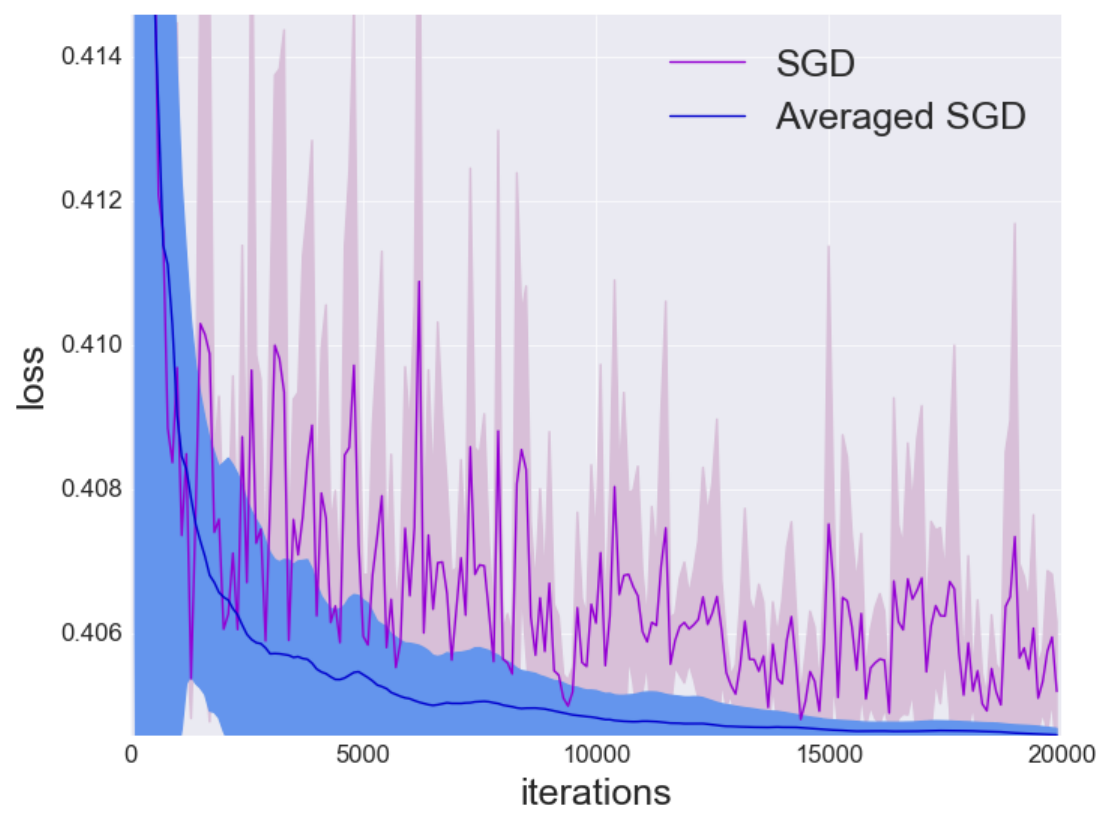}&
  \hspace{-5mm} \includegraphics[angle=0,width=45mm]{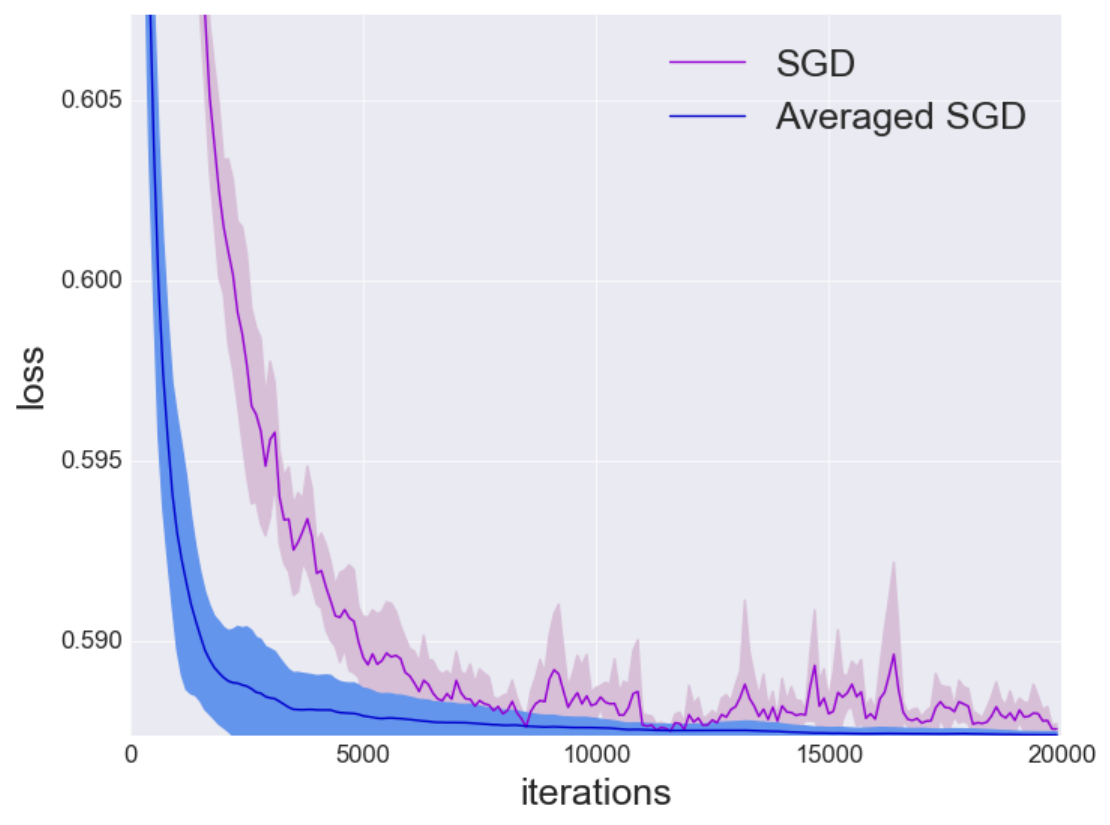}&
  \hspace{-5mm} \includegraphics[angle=0,width=45mm]{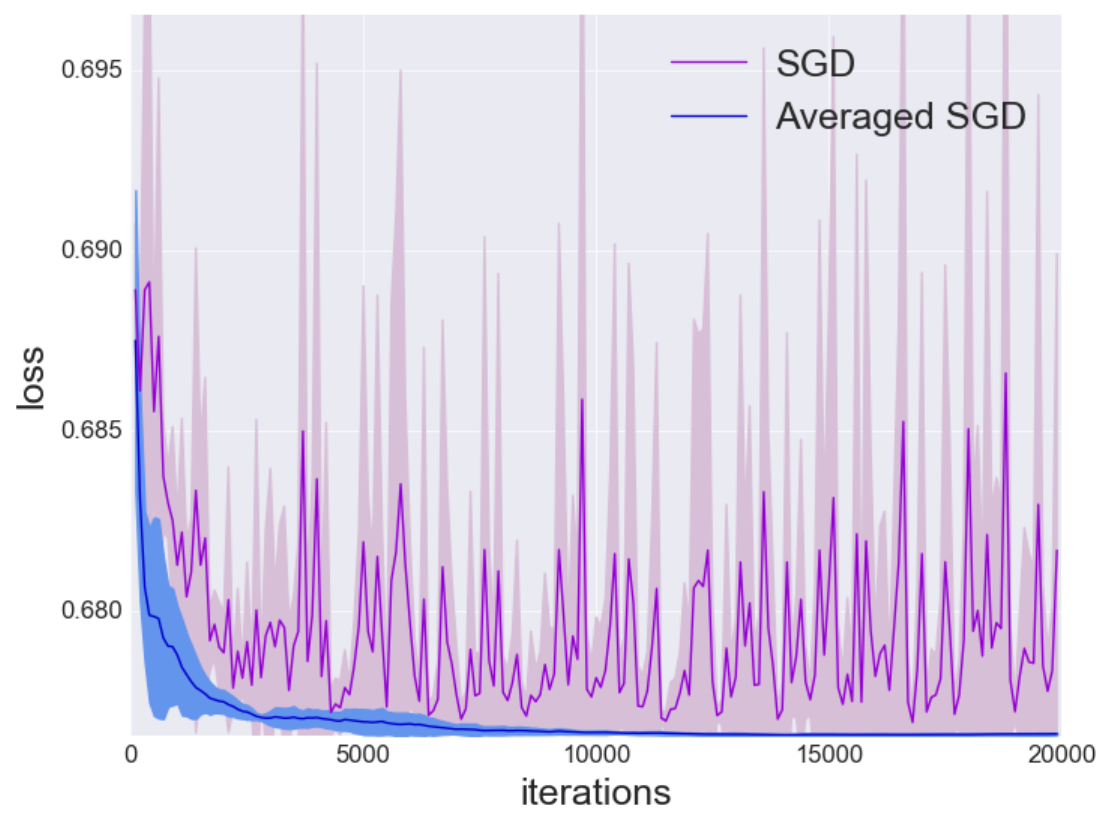}\\
  \hspace{-1mm} \includegraphics[angle=0,width=45mm]{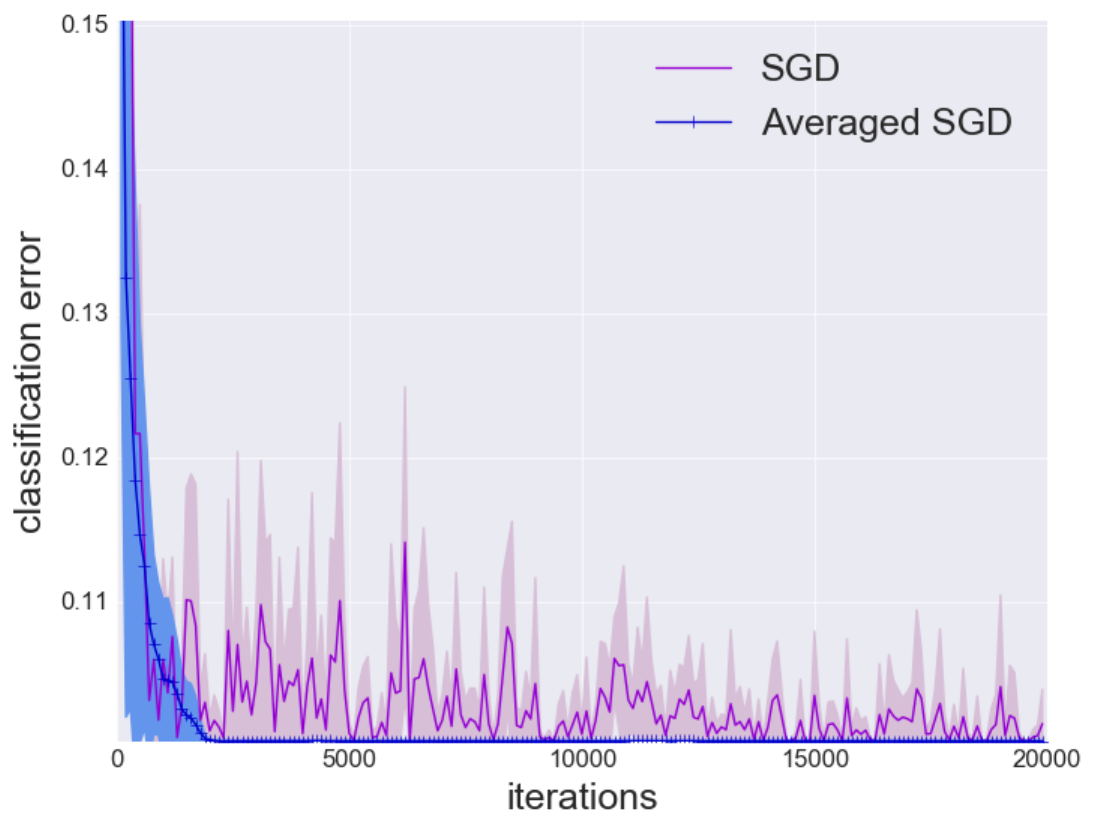}&
  \hspace{-5mm} \includegraphics[angle=0,width=45mm]{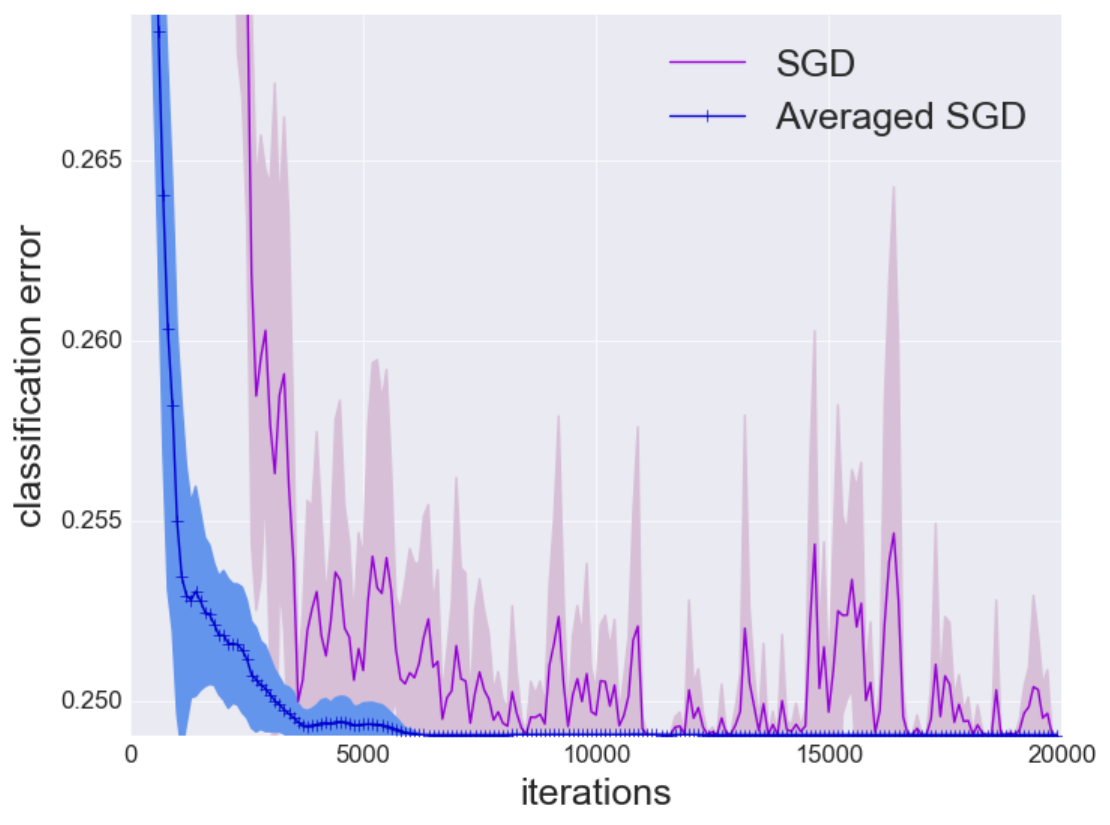}&
  \hspace{-5mm} \includegraphics[angle=0,width=45mm]{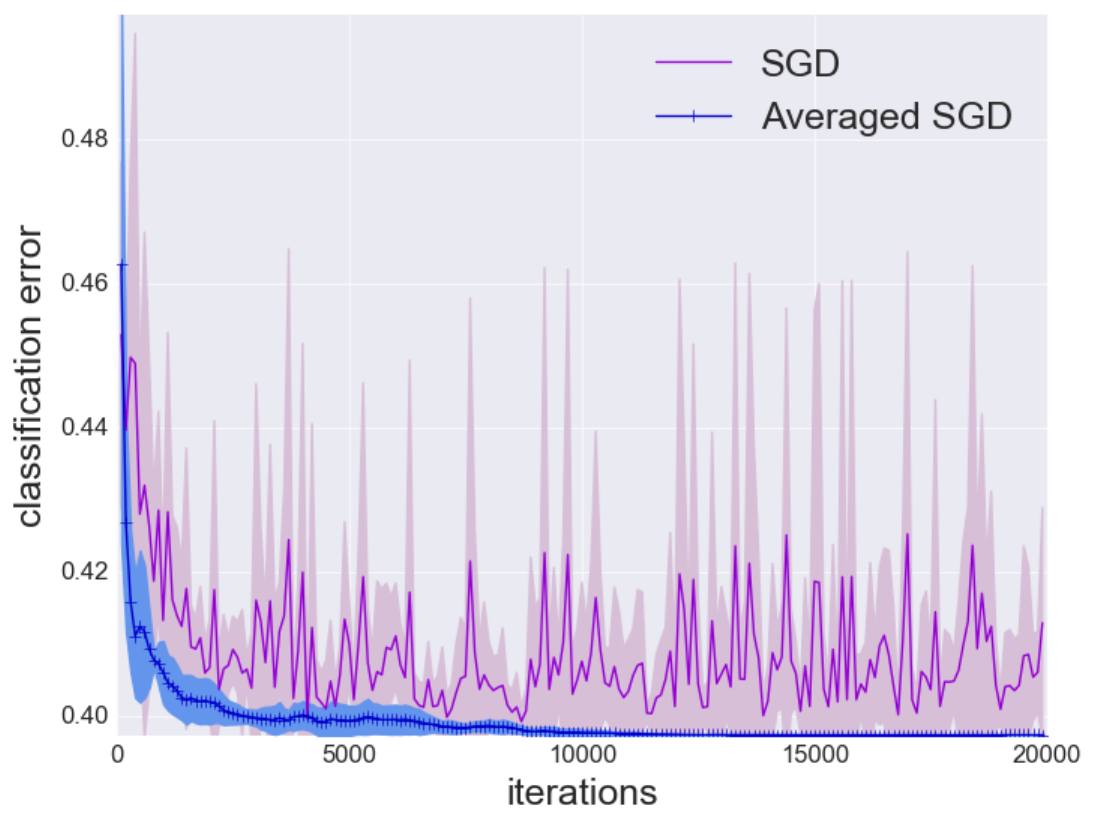}\\
  \hspace{-1mm} \includegraphics[angle=0,width=45mm]{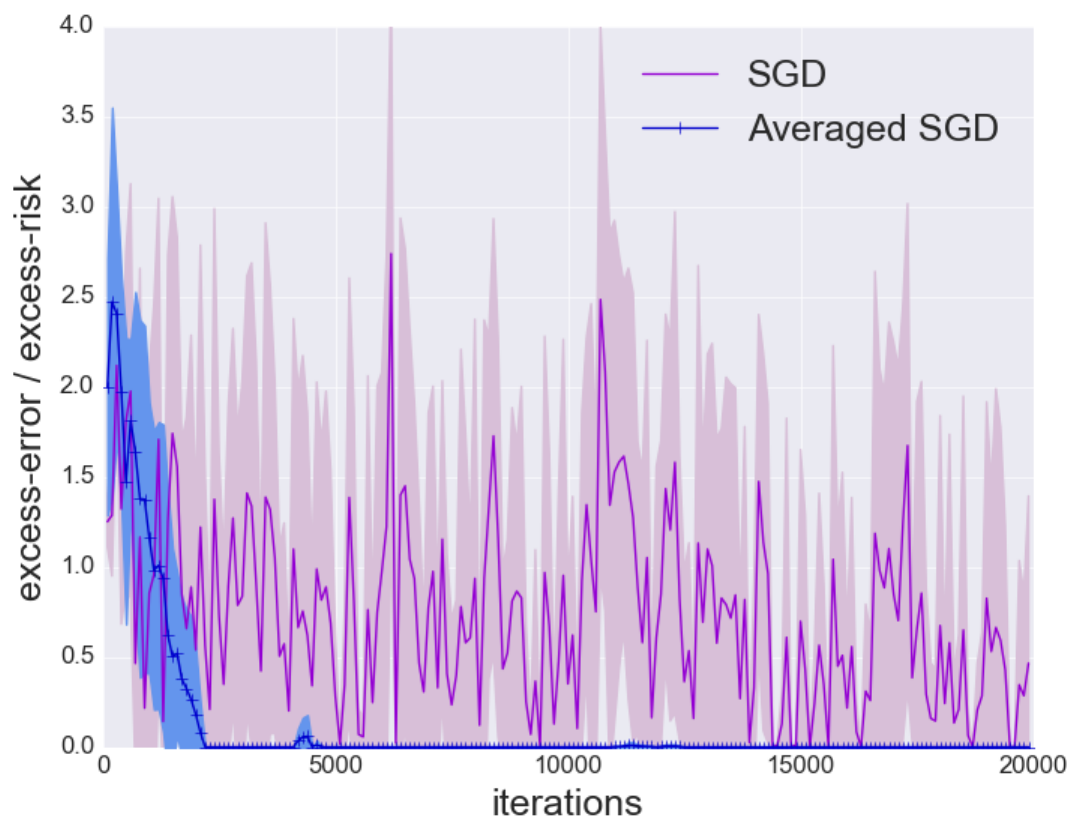}&
  \hspace{-5mm} \includegraphics[angle=0,width=45mm]{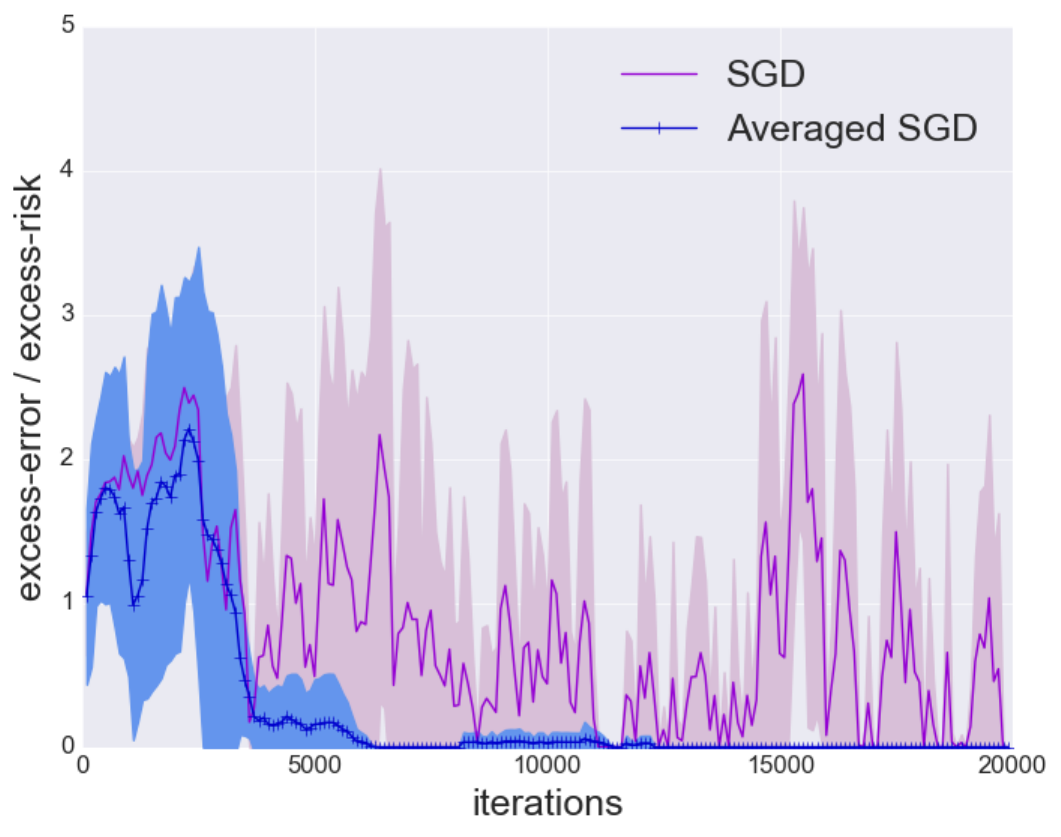}&
  \hspace{-5mm} \includegraphics[angle=0,width=45mm]{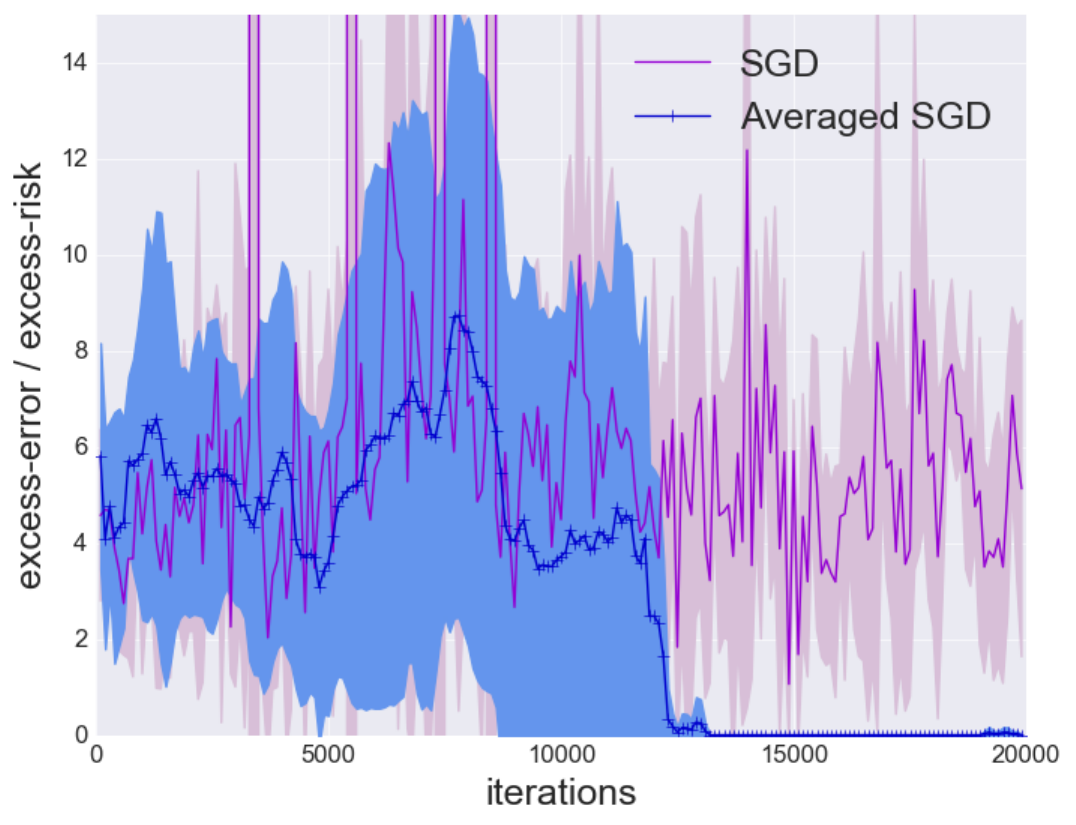}\\       
 \end{tabular}
    }     
    \caption{Learning curves by (averaged) stochastic gradient descent for the binary logistic regression.
    Figure depicts values of loss functions (top), classification errors (middle), and ratios (bottom): excess errors divided by excess risks, for each $\delta$.} 
\label{fig:experimental_results}
\end{center}
\end{figure*}

Experimental results are depicted in Figure \ref{fig:experimental_results}.
The top row shows mean curves of loss functions and the middle row shows mean curves of classification errors with standard deviations obtained by
stochastic gradient descent (purple line) and averaged stochastic gradient descent (blue line).
As seen in Figure \ref{fig:experimental_results}, the bigger low-noise parameter $\delta$ is, the faster the convergence speed of the classification error becomes.
Especially, for the case $\delta=0.4$ of the smallest noise, much faster convergence of the classification error than the loss is observed.
Indeed, Bayes rule for $\error$ is achieved in earlier iterations.
This phenomenon is also confirmed in the bottom row in Figure \ref{fig:experimental_results} that depicts curves of ratios of excess classification errors to excess risks, for each $\delta$.
The fast decreasing of these curves indicate the fast convergence of classification errors.

\section{Conclusion} \label{sec:conclusion}
In this paper, we have shown the exponential convergence property of the expected classification error under a strong low-noise condition, rather than the expected risk for (averaged) stochastic gradient descent.
The main contribution of this work, compared to existing work, is generalizing the loss function to more general differentiable loss functions, 
such as logistic loss, smoothed hinge loss, and exponential loss.
As a result, the acceleration of the method has been shown for typical binary classification problems.
Finally, our analysis has been verified experimentally.
However, some problems are left for future work.
First, we will investigate whether the class of loss functions can be further extended to non-differential functions such as the hinge loss function.
The second is to exclude the assumption that the Bayes rule for the expected risk is included in the given hypothesis class.
Finally, we will explore the convergence speed of more sophisticated methods, such as stochastic accelerated methods and stochastic variance reduced methods
\citep{schmidt2017minimizing,johnson2013accelerating,defazio2014saga,nitanda2014stochastic,allen2017katyusha,murata2017doubly,frostig2015competing} under the strong low-noise condition.
Although these methods have been studied extensively for the empirical risk minimization problems, their performance for the expected risk and the expected classification error are still unclear.

\paragraph{Acknowledgement}
TS was partially supported by MEXT Kakenhi (26280009, 15H05707 and 18H03201), Japan Digital Design and JST-CREST.

%\subsubsection*{Acknowledgements}

% \bibliographystyle{plain}
%\bibliographystyle{apalike}
%\bibliography{ref}

% ----- for Arxiv -----
\bibliographystyle{apalike}

% ------for Arxiv Prep --------
%\bibliographystyle{apalike}
%\bibliography{ref}
% ---------------------

\ifWITHSUPP

\clearpage
\onecolumn
\renewcommand{\thesection}{\Alph{section}}
\renewcommand{\thesubsection}{\Alph{section}. \arabic{subsection}}
\renewcommand{\thetheorem}{\Alph{theorem}}
\renewcommand{\thelemma}{\Alph{lemma}}
\renewcommand{\theproposition}{\Alph{proposition}}
\renewcommand{\thedefinition}{\Alph{definition}}
\renewcommand{\thecorollary}{\Alph{corollary}}
\renewcommand{\theassumption}{\Alph{assumption}}
\renewcommand{\theexample}{\Alph{example}}

\setcounter{section}{0}
\setcounter{subsection}{0}
\setcounter{theorem}{0}
\setcounter{lemma}{0}
\setcounter{proposition}{0}
\setcounter{definition}{0}
\setcounter{corollary}{0}
\setcounter{assumption}{0}
\setcounter{example}{0}

\part*{\Large{Appendix}}
In this appendix, we provide missing proofs in the paper.

\section{Proof of Proposition \ref{prop:bayes_region}}
We show the convergence of $g_{\lambda}$ to the Bayes rule $g_*$ for $\risk$ in $\hilsp_k$.
\begin{proposition} \label{prop:convergence}
  Let $\risk(g)$ be convex with respect to $g$.
  Suppose assumption {\bf(A5)} holds.
  A minimizer $g_\lambda$ of $\risk_{\lambda}$ converges to the Bayes rule $g_*$ in $\hilsp_k$ as $\lambda \rightarrow 0$.
\end{proposition}
\begin{proof}
Let $\{\lambda_i\}_{i = 1,2,\ldots}$ be a positive decreasing sequence tending to zero in $\realsp$.
Let $i,j$ be arbitrary indices such that $i<j$, i.e., $\lambda_i \geq \lambda_j $.
For $g \in \hilsp_k$ satisfying $\|g\|_{\hilsp_k} < \|g_{\lambda_i}\|_{\hilsp_k}$,
by subtracting $\frac{\lambda_i-\lambda_j}{2}\|g_{\lambda_i}\|_{\hilsp_k}^2 > \frac{\lambda_i-\lambda_j}{2}\|g\|_{\hilsp_k}^2$ from
$\risk(g)+\frac{\lambda_i}{2}\|g\|_{\hilsp_k}^2 \geq \risk(g_{\lambda_i})+\frac{\lambda_i}{2}\|g_{\lambda_i}\|_{\hilsp_k}^2$, 
we get
\begin{equation*}
\risk(g)+\frac{\lambda_j}{2}\|g\|_{\hilsp_k}^2 > \risk(g_{\lambda_i})+\frac{\lambda_j}{2}\|g_{\lambda_i}\|_{\hilsp_k}^2.
\end{equation*}
This implies that if $\|g\|_{\hilsp_k} < \|g_{\lambda_i}\|_{\hilsp_k}$, then $g$ is not optimal point of $\risk_{\lambda_j}$, hence, $\|g_{\lambda_j}\|_{\hilsp_k} \geq \|g_{\lambda_i}\|_{\hilsp_k}$.
The boundedness of this sequence is also confirmed because $g_* \in \hilsp_k$ and for $\forall \lambda>0$,
\begin{equation} \label{proof:convergent_sequence}
  \risk(g_*) + \frac{\lambda}{2}\|g_\lambda\|_{\hilsp_k}^2
  \leq \risk(g_\lambda) + \frac{\lambda}{2}\|g_\lambda\|_{\hilsp_k}^2
  \leq \risk(g_*) + \frac{\lambda}{2}\|g_*\|_{\hilsp_k}^2,
\end{equation}
which implies an inequality $\|g_\lambda\|_{\hilsp_k} \leq \|g_*\|_{\hilsp_k}$. % and convergence $\risk(g_\lambda) \rightarrow \risk(g_*)$ as $\lambda \rightarrow 0$.
Namely, $\{ \|g_{\lambda_i}\|_{\hilsp_k}\}_{i=1,2,\ldots}$ is a bounded increasing sequence and has the limit.
On the other hand, $\{ \risk(g_{\lambda_i})\}_{i=1,2,\ldots}$ is a decreasing sequence with the limit corresponding to $\risk(g_*)$.
Indeed, since $\risk(g_{\lambda_j})+\frac{\lambda_j}{2}\|g_{\lambda_j}\|_{\hilsp_k}^2 \leq \risk(g_{\lambda_i})+\frac{\lambda_j}{2}\|g_{\lambda_i}\|_{\hilsp_k}^2$,
we see
\[ 0 \leq \frac{\lambda_j}{2}( \|g_{\lambda_j}\|_{\hilsp_k}^2 - \|g_{\lambda_i}\|_{\hilsp_k}^2) \leq \risk(g_{\lambda_i}) - \risk(g_{\lambda_j}). \]
Moreover, from the inequality (\ref{proof:convergent_sequence}), $\risk(g_{\lambda_i})$ converges to $\risk(g_*)$.

We next show that the convergence of a sequence $\{g_{\lambda_i}\}_{i=1,2,\ldots}$.
From the strong convexity of $\risk_{\lambda_i}(g)$, we have
\[ \risk(g_{\lambda_i}) + \frac{\lambda_i}{2}\| g_{\lambda_i}\|_{\hilsp_k}^2
  + \frac{\lambda_i}{2}\| g_{\lambda_j} - g_{\lambda_i}\|_{\hilsp_k}^2
  \leq \risk(g_{\lambda_j}) + \frac{\lambda_i}{2}\| g_{\lambda_j}\|_{\hilsp_k}^2. \]
Using $\risk(g_{\lambda_j}) \leq \risk(g_{\lambda_i})$, we get
\[ \| g_{\lambda_j} - g_{\lambda_i}\|_{\hilsp_k}^2
  \leq \| g_{\lambda_j}\|_{\hilsp_k}^2 - \| g_{\lambda_i}\|_{\hilsp_k}^2
  \leq 2\| g_* \|_{\hilsp_k} (  \| g_{\lambda_j}\|_{\hilsp_k} - \|g_{\lambda_i}\|_{\hilsp_k} ) . \]
Since, $\{ \| g_{\lambda_i}\|_{\hilsp_k} \}_{i=1,2,\ldots}$ is a convergent sequence, it is also a Cauchy sequence.
As a result, a sequence $\{g_{\lambda_i}\}_{i=1,2,\ldots}$ is Cauchy in $\hilsp_k$ and has a limit point $g_\infty \in \hilsp_k$.
It follows from the continuity of $\risk$ that $\risk(g_\infty)=\lim_{i\rightarrow \infty}\risk(g_{\lambda_i})$.
Recalling $\lim_{i\rightarrow \infty}\risk(g_{\lambda_i}) = \risk(g_*)$ and the uniqueness of the Bayes rule $g_*$, we conclude $g_\infty = g_*$ up to zero measure sets. 
\end{proof}

We now give a proof of Proposition \ref{prop:bayes_region}.
\begin{proof}[ Proof of Proposition \ref{prop:bayes_region} ]
Noting that $g(x)=\pd<g,k(x,\cdot)>_{\hilsp_k}$ for arbitrary function $g \in \hilsp_k$ and $k(x,\cdot) \in \hilsp_k$ by the definition of kernel function, we get
\begin{equation} \label{eq:norm_relationship}
  \|g\|_{L_\infty} = \sup_{x \in \featuresp} |g(x)|
  \leq \|g\|_{\hilsp_k}\|k(x,\cdot)\|_{\hilsp_k} \leq R \|g\|_{\hilsp_k}.
\end{equation}  
Since, $g_\lambda$ converges to $g_*$ in $\hilsp_k$ from Proposition \ref{prop:convergence}, there exists $\lambda>0$ such that
\[ \|g_\lambda - g_*\|_{\hilsp_k} \leq \frac{m(\delta)}{2R}. \]
Thus, for arbitrary $g \in \hilsp_k$ satisfying $\|g - g_\lambda\|_{\hilsp_k} \leq \frac{m(\delta)}{2R}$, we have
\[ \|g - g_*\|_{L_\infty}
  \leq R\|g - g_*\|_{\hilsp_k}
  \leq R\left(\|g - g_\lambda\|_{\hilsp_k} + \|g_\lambda - g_*\|_{\hilsp_k} \right)
  \leq m(\delta). \]
%\begin{comment}
%Noting that $h_*^{-1}$ is $L'$-Lipschitz continuous and $g_*(X) = h_*(\tpr(1|X))$ almost surely as mentioned in the paper by Assumptions {\bf(A4)} and {\bf(A5)},
%we get for $x\in \featuresp$ almost surely,
%\[ |h_*^{-1}(g(x)) - \tpr(1|x)| =  |h_*^{-1}(g(x)) - h_*^{-1}(g_*(x))| \leq L'|g(x)-g_*(x) | \leq \delta. \]
%Hence, $h_*^{-1}(g(x)) - 1/2$ and $\tpr(1|x) - 1/2$ have the same sign by the low-noise condition {\bf(A3)}.
%Thus, we conclude by {\bf(A4)},
%\[ \sgn(g(x)) = \sgn( h_*^{-1}(g(x)) - 1/2 ) = \sgn( \tpr(1|x) - 1/2 ), \]
%\end{comment}
Since, $m(\delta) \leq |g_*(X)|$ almost surely, we get $\sgn(g_*(X))=\sgn(g(X))$ almost surely for $g \in \hilsp_k$ such that 
$\| g-g_\lambda \|_{\hilsp} \leq \frac{m(\delta)}{2R}$,
that is, $g$ is also the Bayes rule for $\error$.
\end{proof}

\section{Proof of Theorem \ref{thm:sgd_exp_convergence}}
In this section, we give proofs of auxiliary statements needed for the main theorem meaning the exponential convergence of stochastic gradient descent.
We here prove convergence of expected functions obtained by stochastic gradient descent.

\begin{proof}[Proposition \ref{prop:convergence_sgd}]
By $(L+\lambda)$-Lipschitz smoothness $\risk_\lambda$, we have
\begin{align} \label{eq:sgd_inequality_smooth}
  \expec[ \risk_\lambda(g_t - \eta_t G_\lambda(g_t, z_t))]  
  &\leq \expec[\risk_\lambda(g_t)] - \eta_t \expec[ \pd< \nabla \risk_\lambda(g_t),G_\lambda(g_t, z_t) >_{\hilsp_k}] + \frac{(L+\lambda)\eta_t^2}{2}\expec\| G_\lambda(g_t,z_t)\|_{\hilsp_k}^2 \notag \\
  &\leq \expec[\risk_\lambda(g_t)] - \eta_t \expec\| \nabla \risk_\lambda(g_t) \|_{\hilsp_k}^2 + \frac{(L+\lambda)\eta_t^2}{2}(\expec\| \nabla \risk_\lambda (g_t)\|_{\hilsp_k}^2 + \sigma^2) \notag \\
  &\leq \expec[\risk_\lambda(g_t)] - \frac{\eta_t}{2} \expec\| \nabla \risk_\lambda(g_t) \|_{\hilsp_k}^2 + \frac{(L+\lambda)\eta_t^2\sigma^2}{2},
\end{align}
where we used $\eta_t \leq 1/(L+\lambda)$ for the last inequality.
On the other hand, by the strong convexity of $\risk_\lambda$, we have for $\forall g \in \hilsp_k$,
\[ \risk_\lambda(g_t) + \pd< \nabla \risk_\lambda(g_t),g-g_t>_{\hilsp_k} + \frac{\lambda}{2}\|g-g_t\|_{\hilsp_k}^2 \leq \risk_\lambda(g). \]
Minimizing both sides with respect to $g$ in $\hilsp_k$, we have
\begin{equation} \label{eq:sgd_inequality_convex}
  \risk_\lambda(g_t) - \frac{1}{2\lambda}\| \nabla \risk_\lambda(g_t)\|_{\hilsp_k}^2 \leq \risk_\lambda(g_\lambda).
\end{equation}
By combining two inequalities (\ref{eq:sgd_inequality_smooth}) and (\ref{eq:sgd_inequality_convex}) and subtracting $\risk_\lambda(g_\lambda)$, we get
\begin{equation} \label{eq:sgd_reduction}
\expec[\risk_\lambda(g_{t+1})] - \risk_\lambda(g_\lambda) \leq (1-\eta_t\lambda)\left( \expec[ \risk_\lambda(g_t)] -\risk_\lambda(g_\lambda)\right) + \frac{(L+\lambda)\eta_t^2\sigma^2}{2}.
\end{equation}

We now show the following convergence rate by induction on $t$.
\begin{equation} \label{eq:convergence_risk}
  \expec[ \risk_\lambda(g_{t}) ] - \risk_\lambda(g_\lambda) \leq \frac{\nu}{\gamma + t}.
\end{equation}
For $t=1$, it is clearly true from the choice of $\nu$.
We suppose that the inequality (\ref{eq:convergence_risk}) is true for $t$.
We denote $\hat{t} = \gamma + t$ for simplicity.
Then, we have that from the inequality (\ref{eq:sgd_reduction}) and $\eta_t=2/\lambda \hat{t}$,
\begin{align*}
  \expec[ \risk_\lambda(g_{t+1}) ] - \risk_\lambda(g_\lambda)
  &\leq \left( 1-\frac{2}{\hat{t}} \right) \frac{\nu}{\hat{t}} + \frac{2(L+\lambda)\sigma^2}{\lambda^2 \hat{t}^2}\\
  &= \frac{(\hat{t}-1)\nu}{\hat{t}^2} - \frac{\nu}{\hat{t}^2} + \frac{2(L+\lambda)\sigma^2}{\lambda^2 \hat{t}^2}\\
  &\leq \frac{\nu}{\hat{t}+1},
\end{align*}
where we used $\hat{t}^2>(\hat{t}+1)(\hat{t}-1)$ and the definition of $\nu$.
Thus, the inequality (\ref{eq:convergence_risk}) is true for all $T \geq 1$.
From the strong convexity and Jensen's inequality for $\risk_\lambda$, we have
\begin{equation*}
  \|\expec[g_t]-g_\lambda\|_{\hilsp_k}^2 \leq \frac{2}{\lambda}(\risk_\lambda(\expec[g_t]) - \risk_\lambda(g_\lambda))
  \leq \frac{2}{\lambda}(\expec[\risk_\lambda(g_t)] - \risk_\lambda(g_\lambda)).
\end{equation*}
This finishes the proof of the proposition.
\end{proof}

As argued in the paper, the proof of Proposition \ref{prop:martingale_bound_sgd} is reduced to bounding $\|g_{T+1} - g_{T+1}^t\|_{\infty}$.
The following proposition is useful for that purpose.
Let $g_s^t$ ($s\geq t \in \{1,\ldots,T+1\}$) be the $s$-th iterate depending on $(Z_1,\ldots,Z_{t-1},Z_t',Z_{t+1},\ldots,Z_s)$.

\begin{proposition} \label{prop:contraction}
  Suppose Assumptions {\bf(A1)} and {\bf(A2)} hold.
  We consider Algorithm \ref{alg:sgd} without the averaging option and with a decreasing learning rates $\eta_t$.
  We assume that $\|g_1\|_{\hilsp_k}\leq (2\eta_1 + 1/\lambda)MR$ and $\eta_1 \leq \min\{1/L, 1/2\lambda\}$.
  Then, for $t \in \{1,\ldots,T\}$, it follows that
\begin{enumerate}
\item $\| g_{t+1} - g_{t+1}^t\|_{\hilsp_k} \leq 6MR\eta_t$,
\item $\| g_{s+1} - g_{s+1}^t\|_{\hilsp_k} \leq (1-\eta_s\lambda)\| g_{s} - g_{s}^t\|_{\hilsp_k}$\ \ for $s\geq t+1$.
\end{enumerate}
\end{proposition}
\begin{proof}
By the assumptions, we find that the stochastic gradient of $l$ in $\hilsp_k$ is bounded as follows:
\[ \| \partial_\zeta l(g(x),y)k(x,\cdot) \|_{\hilsp_k} \leq MR. \]
Therefore, if $\|g_t\|_{\hilsp_k} \geq \frac{1}{\lambda}MR$, then
\begin{align*}
  \|g_{t+1}\|_{\hilsp_k}
  &= \|g_{t} - \eta_t \partial_\zeta l(g(X_t),Y_t)k(X_t,\cdot) - \eta_t \lambda g_t \|_{\hilsp} \\
  &\leq (1-\eta_t \lambda)\|g_{t}\|_{\hilsp_k} + \eta_t MR \\
  &\leq \|g_t\|_{\hilsp_k}.
\end{align*}
This means a generated sequence $\{g_t\}_{t=1,\ldots,T+1}$ is included in a closed ball centered at the origin with radius $(2\eta_1 + 1/\lambda)MR$
as long as an initial function $g_1$ is contained in this ball.
Thus, the norm of $G_\lambda(g_t,Z_t)$ is bounded by $2(1 + \lambda \eta_1)MR \leq 3MR$.

The first statement can be shown as follows: since $g_t = g_t^t$,
\begin{align*}
  \| g_{t+1} - g_{t+1}^t\|_{\hilsp_k}
  = \eta_t \| G_\lambda(g_t,Z_t) - G_\lambda(g_t,Z_t')  \|_{\hilsp_k} 
  \leq 6\eta_tMR.
\end{align*}

We next show the second statement.
The Lipschitz smoothness of $\risk$ leads to the following inequality which can be confirmed by naturally extending the proof of \cite{nes2004} to the Hilbert space.
Let $\partial_g l(g,z)$ denote the gradient of $l(g,z)$ with respect to $g$ in $\hilsp_k$.
Then, we have for $\forall g, \forall g' \in \hilsp_k$, 
\begin{equation} \label{eq:smooth_inequality2}
  \pd< \partial_g l(g,z) - \partial_g l(g',z), g - g' >_{\hilsp_k} \geq \frac{1}{L}\| \partial_g l(g,z) - \partial_g l(g',z)  \|_{\hilsp_k}^2.
\end{equation}
Thus, we have that for $s\geq t+1$,
\begin{align*}
  \| g_{s+1} - g_{s+1}^t\|_{\hilsp_k}^2
  &= \| (1-\eta_s\lambda)(g_{s} - g_{s}^t) - \eta_s( \partial_g l(g_s,Z_s) - \partial_g l(g_s^t,Z_s)) \|_{\hilsp_k}^2 \\
  &= (1-\eta_s \lambda)^2\| g_{s} - g_{s}^t \|_{\hilsp_k}^2 - 2\eta_s(1-\eta_s \lambda)\pd< \partial_g l(g_s,Z_s) - \partial_g l(g_s^t,Z_s) , g_{s} - g_{s}^t> \\
  &+ \eta_s^2\| \partial_g l(g_s,Z_s) - \partial_g l(g_s^t,Z_s) \|_{\hilsp_k}^2 \\
  &\leq (1-\eta_s \lambda)^2\| g_{s} - g_{s}^t \|_{\hilsp_k}^2 - \eta_s\left( \frac{1}{L} - \eta_s\right)\| \partial_g l(g_s,Z_s) - \partial_g l(g_s^t,Z_s) \|_{\hilsp_k}^2 \\
  &\leq (1-\eta_s \lambda)^2\| g_{s} - g_{s}^t \|_{\hilsp_k}^2, 
\end{align*}
where we used the inequality (\ref{eq:smooth_inequality2}) and conditions on learning rates.
\end{proof}

Utilizing this proposition, the stable property of stochastic gradient descent is shown.

\begin{proof}[Proof of Proposition \ref{prop:martingale_bound_sgd}]
From Proposition \ref{prop:contraction}, we immediately obtain the bound: for $t \in \{1,\ldots,T\}$,
\begin{equation} \label{eq:sgd_contraction}
\| g_{T+1} - g_{T+1}^t \|_{\hilsp_k} \leq 6MR\eta_t \prod_{s=t+1}^T(1-\eta_s \lambda).
\end{equation}
From the following inequality,
\[ \prod_{s=t+1}^T(1-\eta_s \lambda)
  = \prod_{s=t+1}^T\frac{\gamma + s - 2}{\gamma + s} < \frac{\gamma+t}{\gamma + T}, \]
where the last inequality hold clearly by expanding the product, 
the right hand side of the inequality (\ref{eq:sgd_contraction}) is upper bounded as follows
\[ 6MR\eta_t \prod_{s=t+1}^T(1-\eta_s \lambda)
  \leq 6MR\eta_t \frac{\gamma + t}{\gamma + T} 
  = \frac{12MR}{\lambda (\gamma + T)}. \]

We finally obtain the desired bound:
\[ \sum_{t=1}^T\| D_t \|_\infty^2
  \leq \sum_{t=1}^T \frac{144M^2R^2}{\lambda^2 (\gamma + T)^2}
  \leq \frac{144M^2R^2}{\lambda^2(\gamma+T)}.\]
\end{proof}

\section{Proof of Theorem \ref{thm:asgd_exp_convergence}}
In this section we provide auxiliary results for showing Theorem \ref{thm:asgd_exp_convergence}.
Using them, we can show the theorem in the same way as in the case of stochastic gradient descent without averaging.

We first give a convergence rate of expected functions obtained by averaged stochastic gradient descent.
Recall that $\overline{g}_{T+1}=\sum_{t=1}^{T+1}\frac{2(\gamma + t - 1)}{(2\gamma + T)(T+1)}g_t$.

\begin{proposition}
  Let the loss function $\phi$ be convex, that is, let $l(g(x),y)$ be also convex with respect to $g$.
  Consider Algorithm \ref{alg:sgd} with the averaging option.
  Learning rates and averaging weights are $\eta_t = 2/\lambda(\gamma + t)$ and $\alpha_t = \frac{2(\gamma+t-1)}{(2\gamma +T)(T+1)}$, respectively.
  Then, it follows that
  \[ \| \expec[ \overline{g}_{T+1}] - g_\lambda \|_{\hilsp_k}^2
    \leq \frac{2}{\lambda}\left( \frac{18M^2R^2}{\lambda (2\gamma + T)}  + \frac{\lambda \gamma(\gamma - 1)}{2(2\gamma + T)(T+1)} \| g_1 - g_\lambda \|_{\hilsp_k}^2 \right).  \]
\end{proposition}
\begin{proof}
Recall that the norm of the stochastic gradient $G_\lambda(g_t,Z_t)$ can be upper-bounded by $3MR$ as shown in the proof of Proposition \ref{prop:contraction}.
Combining this with the strong convexity of $\risk_\lambda$, we have
\begin{align*}
\expec \| g_{t+1}- g_\lambda \|_{\hilsp_k}^2
&=\expec \| g_{t}- g_\lambda \|_{\hilsp_k}^2 - 2\eta_t \expec[\pd< g_t - g_\lambda, G_\lambda(g_t,Z_t)>_{\hilsp_k}]
 + \eta_t^2 \expec\| G_\lambda(g_t,Z_t) \|_{\hilsp_k}^2 \\
&\leq \expec \| g_{t}- g_\lambda \|_{\hilsp_k}^2 - 2\eta_t \expec[\pd< g_t - g_\lambda, \nabla \risk_\lambda(g_t)>_{\hilsp_k}]
 + 9\eta_t^2 M^2R^2 \\
&\leq \expec \| g_{t}- g_\lambda \|_{\hilsp_k}^2
 - 2 \eta_t \left(\expec[\risk_\lambda(g_t)] - \risk_\lambda(g_\lambda) + \frac{\lambda}{2}\expec\| g_t - g_\lambda \|_{\hilsp_k}^2  \right)
 + 9\eta_t^2 M^2R^2 \\      
\end{align*}
Thus, we have
\begin{align*}
  \expec[\risk_\lambda(g_t)] - \risk_\lambda(g_\lambda)
  &\leq \frac{9\eta_tM^2R^2}{2} + \frac{1-\lambda\eta_t}{2\eta_t} \expec \| g_t - g_\lambda \|_{\hilsp_k}^2 - \frac{1}{2\eta_t}\expec \| g_{t+1}- g_\lambda \|_{\hilsp_k}^2 \\
  &= \frac{9M^2R^2}{\lambda(\gamma + t)} + \frac{\lambda(\gamma + t - 2)}{4} \expec \| g_t - g_\lambda \|_{\hilsp_k}^2
    - \frac{\lambda(\gamma + t)}{4}\expec \| g_{t+1}- g_\lambda \|_{\hilsp_k}^2.
\end{align*}
By multiplying $\gamma + t -1$ and taking sum over $t \in \{1,\ldots,T+1\}$, we get
\begin{align*}
  \sum_{t=1}^{T+1} (\gamma + t - 1)(\expec[\risk_\lambda(g_t)] - \risk_\lambda(g_\lambda))
  &< \frac{9M^2R^2T}{\lambda}  + \frac{\lambda}{4}\sum_{t=1}^{T+1} \{ (\gamma + t - 1)(\gamma + t - 2))\expec \| g_t - g_\lambda \|_{\hilsp_k}^2 \\
  &- (\gamma + t)(\gamma + t -1)\expec \| g_{t+1}- g_\lambda \|_{\hilsp_k}^2 \} \\
  &\leq \frac{9M^2R^2(T+1)}{\lambda}  + \frac{\lambda}{4}\gamma(\gamma - 1) \| g_1 - g_\lambda \|_{\hilsp_k}^2.
\end{align*}
Thus, by dividing $(2\gamma + T)(T+1)/2$ and applying Jensen's inequality for $\risk_\lambda$, the following convergence rate is obtained:
\[ \expec \left[ \risk_\lambda\left( \sum_{t=1}^{T+1}\frac{2(\gamma + t - 1)g_t}{(2\gamma + T)(T+1)} \right) - \risk_\lambda(g_\lambda) \right]
  \leq \frac{18M^2R^2}{\lambda (2\gamma + T)}  + \frac{\lambda \gamma(\gamma - 1)}{2(2\gamma + T)(T+1)} \| g_1 - g_\lambda \|_{\hilsp_k}^2. \]
Thus, the desired inequality is obtained by Jensen's inequality and the strong convexity of $\risk_\lambda$.
\end{proof}

\begin{proposition}
  Suppose the same assumptions as in Proposition \ref{prop:contraction}.
  Consider Algorithm \ref{alg:sgd} with the averaging option.
  Learning rates and averaging weights are $\eta_t = 2/\lambda(\gamma + t)$ and $\alpha_t = \frac{2(\gamma+t-1)}{(2\gamma +T)(T+1)}$, respectively.
  Then, it follows that
  \[ \sum_{t=1}^T\| D_t \|_\infty^2 \leq \frac{288 M^2R^2}{\lambda^2(2\gamma + T)}, \]
  where $D_t = \expec[\overline{g}_{T+1} | Z_1,\ldots,Z_t ] - \expec[\overline{g}_{T+1} | Z_1,\ldots,Z_{t-1} ]$.
\end{proposition}
\begin{proof}\
  Note that $\overline{g}_{t+1} \leftarrow (1-\beta_t) \overline{g}_t + \beta_t g_{t+1}$, where $\beta_t = \frac{2(\gamma + t)}{(t+1)(2\gamma + t)}$.
  Thus, we have
  \begin{equation*}
    \| \overline{g}_{T+1} - \overline{g}_{T+1}^t \|_{\hilsp_k} \leq (1-\beta_T)\| \overline{g}_{T} - \overline{g}_{T}^t \|_{\hilsp_k} + \beta_T\| g_{T+1} - g_{T+1}^t \|_{\hilsp_k}.
  \end{equation*}
  By recursively expanding updates, we obtain the following upper-bound:
  \begin{equation*}
    \sum_{s=t}^T\left\{ \prod_{r=s+1}^T(1-\beta_r) \right\} \beta_s \| g_{s+1} - g_{s+1}^t \|_{\hilsp_k}.
  \end{equation*}
  Recall the proof of Proposition \ref{prop:martingale_bound_sgd}, it follows that 
  \begin{equation*}
    \| g_{s+1} - g_{s+1}^t \|_{\hilsp_k} \leq 6MR\eta_t \prod_{r=t+1}^s(1-\eta_r\lambda) \leq \frac{12MR}{\lambda(\gamma + s)}.
  \end{equation*}
  Since $\prod_{r=s+1}^T(1-\beta_r) = \frac{(s+1)(2\gamma + s)}{(T+1)(2\gamma + T)}$, we have
  \begin{equation*}
    \| \overline{g}_{T+1} - \overline{g}_{T+1}^t \|_{\hilsp_k}
    \leq \sum_{s=t}^T \frac{24MR}{\lambda(T+1)(2\gamma + T)}
    = \frac{24MR(T-t+1)}{\lambda(T+1)(2\gamma + T)}.
  \end{equation*}
  Therefore, we have the following bound: since $\sum_{t=1}^T t^2 = T(T+1)(2T+1)/6$,
  \begin{align*}
   \sum_{t=1}^T\| D_t \|_\infty^2 
    &\leq \frac{24^2M^2R^2}{\lambda^2(T+1)^2(2\gamma + T)^2} \sum_{t=1}^T (T-t+1)^2 \\
    &\leq \frac{24\cdot 4M^2R^2(2T+1)}{\lambda^2(2\gamma + T)^2} \\
    &\leq \frac{288 M^2R^2}{\lambda^2(2\gamma + T)}.
  \end{align*}
\end{proof}

\fi

\end{document}